\documentclass{article}

\usepackage[preprint]{neurips_2022}

\usepackage[utf8]{inputenc} 
\usepackage[T1]{fontenc}    
\usepackage[colorlinks=true, linkcolor=blue, citecolor=blue]{hyperref}
\usepackage{url}            
\usepackage{booktabs}       
\usepackage{amsfonts,mathrsfs}       
\usepackage{nicefrac}       
\usepackage{microtype}      
\usepackage{xcolor}         
\usepackage{enumitem}
\usepackage{xspace}
\usepackage{algorithm}
\usepackage[noend]{algorithmic}
\usepackage{amsmath, bm}
\usepackage{def}
\allowdisplaybreaks

\usepackage{chngcntr}
\usepackage{apptools}
\counterwithin{theorem}{section}

\title{Byzantine-Robust Online and Offline Distributed Reinforcement Learning}

\author{%
  Yiding Chen\\
  University of Wisconsin-Madison\\
  \texttt{yiding@cs.wisc.edu} \\
  \And
  Xuezhou Zhang\\
  Princeton University\\
  \texttt{xz7392@princeton.edu} \\
  \And
  Kaiqing Zhang\\
  Massachusetts Institute of Technology\\
  \texttt{kaiqing@mit.edu}
  \And
  Mengdi Wang\\
  Princeton University\\
  \texttt{mengdiw@princeton.edu}
  \And
  Xiaojin Zhu\\
  University of Wisconsin-Madison\\
  \texttt{jerryzhu@cs.wisc.edu} \\
}

\begin{document}

\maketitle

\begin{abstract}
We consider a distributed reinforcement learning setting where multiple agents separately explore the environment and communicate their experiences through a central server. However,
$\alpha$-fraction of agents are adversarial and can report arbitrary fake information.
Critically, these adversarial agents can collude and their fake data can be of any sizes.
We desire to robustly identify a near-optimal policy for the underlying Markov decision process in the presence of these adversarial agents.
Our main technical contribution is \textsc{Weighted-Clique}, a novel algorithm for the \textit{robust mean estimation from batches} problem, that can handle arbitrary batch sizes.
Building upon this new estimator, in the offline setting, we design a Byzantine-robust distributed  pessimistic value iteration algorithm;
in the online setting, we design a Byzantine-robust distributed optimistic value iteration algorithm.
Both algorithms obtain near-optimal sample complexities and achieve superior robustness guarantee than prior works.
\end{abstract}

\section{Introduction}
Distributed learning systems have been one of the main driving force to recent successes of deep learning \citep{verbraeken2020survey,goyal2017accurate,abadi2016tensorflow}.  
Advances in designing efficient distributed optimization algorithms\citep{horgan2018distributed} and deep learning infrastructures \citep{espeholt2018impala} have enables training powerful models with hundreds of billions of parameters \citep{brown2020language}. However, with the outsourcing of computation and data collection, new challenges emerges.
In particular, distributed system has been found vulnerable to Byzantine failure \citep{lamport1982byzantine}, meaning there could be agents with failure that may send arbitrary information to the central server.
Even a small number of Byzantine machines who send out moderate value can lead to a significant loss in performance \citep{yin2018byzantine,ma2019policy, zhang2020adaptive}, which raise security concern in real world applications such as chatbot \citep{neff2016automation} and autonomous vehicles \citep{eykholt2018robust, ma2021adversarial}.
In addition, other desired properties are chased after, such as protecting data privacy of individual data contributors \citep{sakuma2008privacy, liu2019privacy} and and reducing communication cost\citep{dubey2021provably}.
These challenges requires new algorithmic design on the server side, which is the main focus of this paper.

When it comes to reinforcement learning (RL), distributed learning has been prevalent to many large scale decision making problems even before the deep learning era, such as cooperative learning in robotics systems \citep{ding2020distributed}, power grids optimization \citep{yu2014multi}, automatic traffic control \citep{bazzan2009opportunities}. Different from supervised learning where the data distribution of interest is often fixed a prior, reinforcement learning requires active exploration on the agent's side to discover the optimal policy for the current task, thus creating new challenges in achieving the above desiderata while exploring in an unknown environment. 

This paper studies this precise problem:
\begin{center}\emph{Can we design a distributed RL algorithm that is sample efficient and robust to Byzantine agents, while having small communication costs and promote data privacy?}\end{center}

We study Byzantine-robust RL in both online and offline reinforcement learning settings:
in the online setting, a central server is designed to outsource the exploration task to $m$ agents, and the agents collect experiences and send back to the server, and the server use the data to update its policy;
in the offline setting, a central server collects logged data from $m$ agents and use the data to identify a good policy, without the additional interaction to the environment. 
Importantly, among the $m$ agents, $\alpha$ fraction are Byzantine, meaning they are allowed to send out arbitrary data in both the online and offline setting.
We summarize our contributions as following:
\begin{enumerate}[leftmargin=*,itemsep=0pt]
\item We design \textsc{Weighted-Clique}, a robust mean estimation algorithm for learning from batches. 
By utilizing the batch structure, the estimation error of our algorithm vanishes with more data.
Compared to prior works \citep{qiao2017learning,chen2020efficiently,jain2021robust, yin2018byzantine}, our algorithm adapts to arbitrary batch sizes, which is desired in many applications of interests.
\item We design \textsc{Byzan-UCBVI}, a Byzantine-Robust variant of optimistic value iteration for online RL by calling \textsc{Weighted-Clique} as a subroutine.
We show that \textsc{Byzan-UCBVI} achieves near-optimal regret with $\alpha$-fraction Byzantine agents.
Meanwhile, \textsc{Byzan-UCBVI} also enjoys a logarithmic communication cost and 
switching cost \citep{bai2019provably,zhang2020almost,gao2021provably}, and preserves data privacy of individual agents.
\item We design \textsc{Byzan-PEVI}, a Byzantine-Robust variant of pessimistic value iteration for offline RL again utilizing \textsc{Weighted-Clique} as a subroutine.
Despite of the presence of Byzantine agents, 
we show that \textsc{Byzan-PEVI} can learn a near-optimal policy with polynomial name of samples, when certain good coverage properties are satisfied \citep{zhang2021corruption}.
\end{enumerate}

\section{Related Work}
\paragraph{Reinforcement learning:}
Reinforcement learning studies the optimal strategy in a Markov Decision Process (MDP) \citep{sutton2018reinforcement}. 
\citep{azar2017minimax, dann2017unifying} show that UCB style algorithm achieves minimax regret bound in tabular MDPs. 
Recent work extend the theoretical understanding to RL with function approximation \citep{jin2020provably, yang2020reinforcement}.
\citep{jin2021pessimism, rashidinejad2021bridging} use pessimistic strategy to efficiently learn a nearly optimal policy in offline setting. 
Recently, \citep{bai2019provably,zhang2020almost,gao2021provably} study low switching cost RL algorithm, meaning the learning agent has small budget for policy changes.

\paragraph{Distributed reinforcement learning:}
Parallel RL deploys large-scale models in distributed system \citep{kretchmar2002parallel}. 
\citep{horgan2018distributed, espeholt2018impala} provide distributed architecture for deep reinforcement learning by parallelizing the data generating process.
\citep{dubey2021provably, agarwal2021communication, chen2021communication} provide the first sets of theoretical guarantee for performance and communication cost in parallel RL.

\paragraph{Robust statistics:} Robust statistics has a long history \citep{huber1992robust, tukey1960survey}, which studies learning with corrupted dataset. 
In modern machine learning, models are high dimensional. 
Recent work provide sample and computationally efficient algorithms for robust mean and covariance estimation in high dimension \citep{diakonikolas2019recent, lai2016agnostic}. 
Shortly after, those robust mean estimators are applied to robust supervised learning \citep{diakonikolas2019sever, prasad2018robust} and RL \citep{zhang2021corruption,zhang2021robust}.   
\paragraph{Robust learning from batches:}
Another line of work studies robust learning from batches \citep{qiao2017learning,chen2020efficiently,jain2021robust,yin2018byzantine}.
A collection of data is generated from data sources while a fraction of the data sources are corrupted. 
By exploiting the batch structure of the data, these algorithms achieve significantly high accuracy than non-batch setting \citep{diakonikolas2019recent}.
To our best knowledge, all of these work study batches with equal size.
Our paper is the first that generalizes to the setting where the batch size varies. Similarly, these works all assume the same batch sizes from each agent, which may not be true in many crowd-sourcing applications.

\paragraph{Byzantine-robust distributed learning:}
Byzantine-Robust learning algorithm studies learning under Byzantine failure \citep{lamport1982byzantine}.
\citep{chen2017distributed} provides a Byzantine gradient descent via geometric median of mean estimation for the gradients.
\citep{yin2018byzantine} provides robust distributed gradient descent algorithms with optimal statistics rates.

\paragraph{Corruption robust RL and Byzantine-robust RL:}
There is a line of work studying adversarial attack against reinforcement learning \citep{ma2019policy,zhang2020adaptive, huang2017adversarial}, and corruption robust  reinforcement RL for online \citep{zhang2021robust, lykouris2021corruption} and offline \citep{zhang2021corruption} settings.
\citep{jadbabaie2022byzantine} studies Byzantine-Robust linear bandits in federated setting. Unlike our setting, they allow different agents to be subject to Byzantine attack in different episodes. Our algorithm enjoys a better regret bound and communication cost.
\citep{fan2021fault} provides a Byzantine-robust policy gradient algorithm that is guaranteed to converge to an approximate stationary point while we focus the regret of the algorithm.
\citep{dubey2020private} studies Byzantine-Robust multi-armed bandit, where the corruption can only come from a fixed distribution.
We study a more difficult MDP setting and allow the corruption to be arbitrary.

\section{Robust Mean Estimation from Untruthful Batches}
\label{sec:robust_mean}
We first present our novel algorithm, \textsc{Weighted-Clique}, for the \textit{robust mean estimation from batches} problem, which we define below. \textsc{Weighted-Clique} will be the main workhorse later in our algorithms for both offline and online Byzantine-robust RL problems.

\begin{definition}[Robust mean estimation from batches]\label{def:rb ln fm bt}
There are $m$ data providers indexed by: $\{1, 2, \ldots, m\} =: [m]$. 
Among these providers, we denote the indexes of uncorrupted providers by $\Gcal$
and the indexes of corrupted providers by $\Bcal$,
where $\Bcal \cup \Gcal = [m]$, $\Bcal \cap \Gcal = \emptyset$,  $\left|\Bcal\right| = \alpha m$. 
Any uncorrupted providers has access to a sub-Gaussian distribution $\Dcal$ with mean $\mu$ and variance proxy $\sigma^2$ 
(i.e. $
\EE_{X\sim \mathcal{D}}[X] = \mu
$
and
$
\EE_{X \sim \mathcal{D}}\left[\exp\left(s \left(X - \mu\right)\right)\right] \le    \exp\left({\sigma^2 s^2}/{2}\right)
$,
$
\forall s \in \RR.
$).
For each $j \in [m]$, a data batch 
$\left\{x_j^i\right\}_{i=1}^{n_j}$
is sent to the learner, where $n_j$ can be arbitrary.
For $j\in\Gcal$, 
$\left\{x_j^i\right\}_{i=1}^{n_j}$
are i.i.d. samples drawn from $\Dcal$;
for $j \in \Bcal$, 
$\left\{x_j^i\right\}_{i=1}^{n_j}$
can be arbitrary.
\end{definition}
\pref{def:rb ln fm bt} considers a robust learning problem from batches where we allow arbitarily different batch sizes.  In contrast, prior works  \citep{qiao2017learning,chen2020efficiently,jain2021robust} have only studied the setting with (roughly) equal batch sizes, which is much more restricted.
For this problem, we propose the \textsc{Weighted-Clique} algorithm (\pref{alg:int}). 
Given the batch datasets $\left\{x_1^i\right\}_{i=1}^{n_1}, \left\{x_2^i\right\}_{i=1}^{n_2}, \ldots, \left\{x_m^i\right\}_{i=1}^{n_m}$, parameter $\sigma$ of the sub-Gaussian distribution, corruption level $\alpha$,
and confidence level $\delta$,
\textsc{Weighted-Clique} first performs a clipping step (\pref{ln:clip}) to clip the sizes of the largest $2\alpha m$ batches to the size of the $(2\alpha+1) m$-th largest batch. This is to reduce the impact of corrupted batch on the weighted average in \pref{ln:wted mean}.
Next, a set confidence intervals for the true mean $\mu$ is constructed in \pref{ln:CI} based on the data of each batch, where $I_j = \RR$ if $n_j = 0$.
In order to remove the outliers, the algorithm find the largest set of batches whose confidence intervals all intersect. This can be formulated as a maximum-clique problem, and thus the name \textsc{Weighted-Clique}.
The largest clique can be found efficiently by sorting and scanning endpoints of all $I_j$'s. 
This algorithm returns the weighted average of empirical means of the maximum clique, where the weights are given by clipped sample size, $\tilde n_j$.
\begin{algorithm}
\caption{\textsc{Weighted-Clique}} \label{alg:int}
\begin{algorithmic}[1]
\REQUIRE $\left\{\left\{x_j^i\right\}_{i=1}^{n_j}\right\}_{j \in [m]}$, $\sigma$, $\alpha$, $\delta > 0$
\STATE $\hat x_j \gets \frac{1}{n_j}\sum_{i=1}^{n_j}x_j^i$, for all $j \in [m]$
\STATE $\ncut \gets \mbox{$(2\alpha m+1)$-th largest value in $\{n_j\}_{j\in[m]}$}$  \label{ln:n cut}
\FOR{$j = 1,2,\ldots,m$}
\STATE \label{ln:clip} 
$\tilde n_{j} \gets \min(n_j, \ncut)$ 
\STATE \label{ln:CI} $I_j \gets \left[\hat x_j - \frac{\sigma}{\sqrt{\tilde n_j}}\sqrt{2\log \frac{2m}{\delta}},\hat x_j + \frac{\sigma}{\sqrt{\tilde n_j}}\sqrt{2\log \frac{2m}{\delta}}\right]$
\ENDFOR
\STATE \label{ln:cliq} $U^* \gets \argmax_{U \mbox{ s.t. }  \emptyset\neq\bigcap_{j \in U}I_j }|U|$
\RETURN \label{ln:wted mean}$
\hat x \gets \frac{1}{\sum_{j \in U^*}\tilde n_j}\sum_{j \in U^*} \tilde n_j \hat x_j, \quad \operatorname{Error} \gets \mbox{RHS of \pref{eq:err bound}}    
$
\end{algorithmic}
\end{algorithm}

Intuitively, by choosing the maximum clique in \pref{ln:cliq}, \pref{alg:int} finds a cluster of good data batches and drops extreme batches. 
We show that \pref{alg:int} achieves the following guarantee.
\begin{theorem}\label{thm:weighted int}
Under \pref{def:rb ln fm bt}, if $\ncut > 0$, $\alpha < \frac{1}{2}$, with probability at least $1-\delta$, $\hat x$ returned by \pref{alg:int} satisfies:
\begin{align}
\left|\hat x - \mu\right| 
\le & \frac{2}{\sqrt{\sum_{j \in [m]}\tilde n_j}}
\sigma\sqrt{2\log\frac{2}{\delta}} 
+ \frac{8\alpha m\sqrt{\ncut}}{\sum_{j \in [m]}\tilde n_j}
\sigma\sqrt{2\log\frac{2m}{\delta}}
\label{eq:err bound}
\end{align}
where $\ncut$ and $\tilde n_j$'s are defined in \pref{ln:n cut} and \pref{ln:clip} in \pref{alg:int}.
\end{theorem}
A number of immediate remarks are in order.
\begin{remark}
Note that comparing to prior works \citep{qiao2017learning,chen2020efficiently,jain2021robust}, we allow arbitrary batch sizes. Even if some agents report $n_j = 0$, as long as $\ncut >0$, i.e. there are at least $2\alpha m + 1$ agents reporting non-zero $n_j$'s, our estimator will have a well-behaved error bound. This means that the breakdown point (in the sense of fraction of bad agent) of our algorithm is $\frac{1}{2}$, which is optimal.
\end{remark}
\begin{remark}[Equal batch size case]
When $n_1 = \cdots = n_m=n$, the right hand side of \pref{eq:err bound} becomes
$
 O\left(\frac{\sigma}{\sqrt{n}}\left(\frac{1}{\sqrt{m}} + \alpha\sqrt{\log m}\right)\right).
$
This recovers the rate in \citep{yin2018byzantine}, which is optimal (up to logarithmic factors).
\end{remark}
\begin{remark}[Robust mean estimation v.s. robust mean  estimation from batches]
In classical robust mean estimation setting \citep{huber1992robust,diakonikolas2017being}, the optimal error rate is $O\left(\sigma\left(\alpha + \frac{1}{\sqrt{m}}\right)\right)$ given $m$ total samples and $\alpha$ faction corrupted samples. 
In contrast, due to having data source ID, i.e. the batch indices, the adversary are much restricted. To see this, notice that the equal batch setting can be viewed as robust mean estimation from $m$ data points $\hat x_j$'s. When the batch size $n$ becomes larger, $\hat x_j$ has a smaller variance $\frac{\sigma^2}{{n}}$ and thus the error of robust mean estimation becomes $O\left(\frac{\sigma}{\sqrt{n}}\left(\alpha + \frac{1}{\sqrt{m}}\right)\right)$, which matches the above rate (up to logarithmic factors).
\end{remark}
\begin{remark}[Impossibility result]\label{rem:imposs lg bt}
Our bound in \pref{eq:err bound} does not depend on the largest $2\alpha m$ $n_j$'s. 
This means even if some of the clean agents have infinite samples, the algorithm cannot get a very low error. 
This might look not ideal at first glance, but we show that this is inevitable information-theoretically. Interested readers are referred to \pref{thm:impsb}.
\end{remark}

\begin{remark}[Perturbation stability of the estimator and adaption to distributed setting]
When the good data batch is subject to point-wise perturbation of magnitude at most $\epsilon$, a variant of \prettyref{alg:int} (\pref{alg:int pert} \textsc{Pert-Weighted-Clique}, see \pref{sec:adp pert dist}) suffers at most a $2\epsilon$ term in the error upper bound in addition to \prettyref{eq:err bound}. 
\pref{alg:int} does not need the exact dataset as input,
but only the empirical mean and batch sizes of each data batch.
As we see later, this property is essential to achieve low communication cost and preserve data privacy.
\end{remark}

\section{Byzantine-Robust Learning in Parallel MDP}
We study the problem of Byzantine-robust reinforcement learning in the parallel \emph{Markov Decision Processes} (MDPs) setting with one central server and $m$ agents, $\alpha$ fraction of which may suffer Byzantine failure.
We postpone the precise interaction protocols between the server and agents to \pref{sec:ucb vi} and \pref{sec:pevi}.

In both online and offline settings, we consider a finite horizon episodic tabular Markov Decision Process (MDP)
$\Mcal = \left(\Scal, \Acal, \Pcal, \Rcal, H, \mu_1\right)$.
Where $\Scal$ is the finite state space with $|\Scal|=S$;
$\Acal$ is the finite action space with $|\Acal|=A$;
$\Pcal = \left\{P_h\right\}_{h=1}^H$ is the sequence of transition probability matrix, meaning
$\forall h \in [H]$, $P_h: \Scal \times \Acal \mapsto \Delta(\Scal)$ and $P_h(\cdot | s,a)$ specifies the state distribution in step $h+1$ if action $a$ is taken from state $s$ at step $h$;
$\Rcal = \left\{R_h\right\}_{h=1}^H$ is the sequence of bounded stochastic reward function, meaning $\forall h \in [H]$, $R_h(s,a)$ is the stochastic reward bounded in $[0,1]$ associated with taking action $a$ in state $s$ at step $h$;
$H$ is the time horizon;
$\mu_1$ is the initial state distribution.
For simplicity, we assume $\mu_1$ is deterministic, and has probability mass $1$ on state $s_1$.

Within each episode, the MDP starts at state $s_1$. At each step $h$, the agent observes current state $s_h$ and take an action $a_h$ and receives a stochastic reward $R_h(s_h,a_h)$. After that, the MDP transits to a next state $s_{h+1}$, which is drawn from $P_h(\cdot|s,a)$.
The episode terminates after the agent takes action $a_H$ in state $s_H$ and receives reward $R_H(s_H,a_H)$ at step $H$.

A policy $\pi$ is a sequence of functions $\left\{\pi_1, \ldots, \pi_H\right\}$, each maps from state space $\Scal$ to action space $\Acal$. 
The value function $V_h^{\pi}: \Scal \mapsto [0, H-h+1]$, is the expected sum of future rewards by taking action according policy $\pi$, i.e.
$
V_h^{\pi}(s) := \EE\left[\left.\sum_{t = h}^H R_t(s_t,\pi_t(s_t))\right|s_h = s \right],
$
where the expectation is w.r.t. to the stochasticity of state transition and reward in the MDP.
Similarly, we define the state-action value function $Q_h^{\pi}:\Scal\times\Acal \mapsto [0, H-1+1]$:
$
Q_h^{\pi}(s,a):= \EE\left[R_h(s,a)\right] + 
\EE\left[\left.\sum_{t = h+1}^H R_t(s_t,\pi_t(s_t))\right|s_h = s , a_h = a\right]
$
Let $\pi^* = \left\{\pi^h\right\}$ be an optimal policy and let $V_h^*(s) := V_h^{\pi^*}(s,a)$, $Q_h^*(s) := Q_h^{\pi^*}(s,a)$, $\forall h,s,a$. 

For any $f:\Scal \mapsto [0,H]$, We define the Bellman operator by:
$
\left(\Bell_h f\right)(s,a) = \EE\left[R_h(s,a)\right] + \EE_{s' \sim P_h(\cdot | s,a)}[f(s')]
$
Then the Bellman equation is given by:
\begin{equation}
V_h^{\pi}(s) = Q_h^{\pi}(s,\pi_h(s)), \quad
Q_h^{\pi}(s,a) = \left(\Bell_h V_{h+1}^{\pi}\right)(s,a), \quad
V_{H+1}^{\pi}(s) = 0
\end{equation}
The Bellman optimality equation is given by:
\begin{equation}
V_h^{*}(s) = \max_{a \in \Acal}Q_h^{*}(s,a), \quad
Q_h^{*}(s,a) = \left(\Bell_h V_{h+1}^{*}\right)(s,a), \quad
V_{H+1}^{*}(s) = 0
\end{equation}
We define the state distribution at step $h$ by following policy $\pi$ as 
$d_h^{\pi}(s) := P_h^{\pi}(s_h = s)$,
and the state trajectory distribution of $\pi$ as:
$d^{\pi} := \left\{d_h^{\pi}\right\}_{h=1}^H$.
The goal is to find a policy that maximizes the reward, i.e. find a $\hat \pi$, s.t. $V_1^{\hat\pi}(s_1) = V_1^{*}(s_1) = \max_{\pi}V_1^{\pi}(s_1)$.
To measure the performance of our RL algorithms,
we use suboptimality as our performance metric for offline setting and 
use regret as our performance metric for online setting.
We formalize these two measures in their corresponding sections below.

\section{Byzantine-Robust Online RL \label{sec:ucb vi}}
In the online setting, we assume that a central server and $m$ agents aim to collaboratively minimizing their total regrets.
The agents and server collaborate by following a communication protocol to decide when to synchronize and what information to communicate.
Unlike standard distributed RL setting, we assume $\alpha$ fraction of the agents are Byzantine:
\begin{definition}[Distributed online RL with Byzantine corruption]\label{def:byzan online}
There are $m$ agents consists of two types:
\begin{itemize}[leftmargin=*,itemsep=0pt]
\item \textbf{$(1-\alpha)m$ good agents, denoted by $\Gcal$}: Each of the good agents interacts with a copy of $\Mcal$ and communicates its observations to the server following the interaction protocol;
\item \textbf{$\alpha m$ bad agents, denoted by $\Bcal$}: 
The bad agents is allowed to send arbitrary observations to the server at the end of each episode.
\end{itemize}
\end{definition}

Because the server has no control over the bad agents, we only seek to minimize the error incurred by the good agents.
Formally, we use regret as our performance measure for the online RL algorithm:
\begin{equation}
\operatorname{Regret}(K) = \sum_{k=1}^K \sum_{j\in\mathcal{G}} \left(V_1^*(s_1) - V_1^{\pi_k^j }(s_1)\right), 
\end{equation}
where $\pi_k^j$ is the policy used by agent $j$ in episode $k$.
At the same time , because of the distributed nature of our problem, we want to synchronize between the servers and agents only if it is necessary to reduce the communication cost.

Based on these considerations, we propose the \textsc{Byzan-UCBVI} algorithm (\pref{alg:UCBVI}).
We highlight the following key features of \textsc{Byzan-UCBVI}:
\begin{enumerate}[leftmargin=*,itemsep=0pt]
    \item \textbf{Low-switching-cost style algorithm design}: the server will check the synchronization criteria in \pref{ln:sync check} when receiving requests from agents.
    Each good agent will request synchronization if and only if any of their own $(s,a,h)$ counts doubles (\pref{ln:agent sync}). Importantly, our agents do not need to know other agents' $(s,a,h)$ counts to decide if synchronization is necessary. This design choice reduces the number of policy switches, synchronization rounds and communication cost all from $O(T)$ to $O(\log T)$. Compared to the $O(\sqrt{T})$ communication steps in \citep{jadbabaie2022byzantine}, ours is much lower. Unlike \citep{dubey2021provably}, our agents do not need to know other agents' transition counts in order to decide whether to synchronize or not.
    \item \textbf{Homogeneous policy execution}: In any episode $k$, our algorithm ensures that all good agents are running the same policy $\pi_k$. This ensures that the robust mean estimation achieves the smallest estimation error. Recall that the samples in the large batches are wasted if the batch sizes are severely imbalanced (cf. \pref{sec:robust_mean}).
    \item \textbf{Robust UCBVI updates}: During synchronization, the central server performs policy update using a variant of the UCBVI algorithm \citep{azar2017minimax}: for $h = H, H-1, \ldots, 1$, compute:
    \begin{align}
    &\bar Q_h(\cdot, \cdot) = \left(\hat \Bell_h \hat V_{h+1}\right)(\cdot, \cdot) + \Gamma_h(\cdot,\cdot)\label{eq:ucb vi 1}, \quad \hat Q_h(\cdot, \cdot) = \min\left\{\bar Q_{h}(\cdot,\cdot), H-h+1\right\}^+\\
    &\hat \pi_h(\cdot) = \argmax_{a} \hat Q_h(\cdot, a), \quad
    \hat V_h(\cdot) = \max_{a} \hat Q_h(\cdot, a) \label{eq:ucb vi 4}
    \end{align}
    We replace the empirical mean estimation with our \textsc{Pert-Weighted-Clique} (\textsc{PWC}) algorithm (\pref{alg:int pert}) and design a new confidence bonus accordingly. Instead of estimating the transition matrix and reward function, we directly estimate the Bellman operator given an estimated value function $\hat V_{h+1}$. The server gathers the sufficient statistics from agents in \pref{ln:gather eval}.
\end{enumerate}

\begin{algorithm}
\caption{$\textsc{Byzan-UCBVI}\left(K, \delta,\alpha\right)$}
\label{alg:UCBVI}
\begin{algorithmic}[1]
\STATE [S]$\hat V_{H+1}(\cdot)\gets 0$,
$\hat Q_{H+1}(\cdot, \cdot)\gets 0$,
$\operatorname{SyncCount}_j \gets -1, \forall j \in [m]$,
$\operatorname{Sync}_j \gets \operatorname{TRUE}, \forall j \in [m]$
$\delta'\gets\frac{\delta}{(SAHKm)^{3S}}$,
$\epsilon \gets\frac{1}{SAHKm}$ \COMMENT{We use [S] to denote the action of central server}
\STATE [A]$N_h^{j}(s,a) \gets 0$, $D_h^j \gets \emptyset$, $\forall (j,h,s,a) \in [m] \times [H]\times\mathcal{S} \times \mathcal{A}$ \COMMENT{We use [A] to denote the action of agents}
\FOR{episode $k \in [K]$}
\STATE [S] Receive $\operatorname{Sync}_1, \operatorname{Sync}_2, \ldots, \operatorname{Sync}_m$
\FOR{agent $j \in [m]$}
\IF{$\operatorname{Sync}_j = \operatorname{TRUE}$
and $\operatorname{SyncCount}_j \le SAH\log_2 K$ \label{ln:sync check}}
\STATE [S] $\operatorname{SyncCount}_j \gets \operatorname{SyncCount}_j + 1$,\;\;\;
$\operatorname{SYNCHRONIZE} \gets \operatorname{TRUE}$
\ENDIF
\ENDFOR
\IF{$\operatorname{SYNCHRONIZE}$}
\STATE [A] $N_{h,j}^{\operatorname{old}}(s, a) \gets N_h^j(s,a), \forall s,a,h,j$
\FOR{$h = H, H-1,\ldots,1$}
\STATE [S] Communicate $\hat V_{h+1}(\cdot)$ to each agent
\FOR{$(s,a) \in \mathcal{S} \times \mathcal{A}$}
\STATE [A] $\forall j \in [m]$, Send 
$x_j, n_j \gets \frac{1}{N_h^j(s,a)}\sum_{(s,a, r,s') \in D_h^j}r + \hat V_{h+1}(s'), N_h^j(s,a)$ to Server.\label{ln:gather eval}
\STATE [S] \label{ln:call pwc}
$
\left(\hat \Bell_h \hat V_{h+1}\right)(s,a), \Gamma_h(s,a) 
\gets 
\textsc{PWC}\left(\left\{\left(
x_j,n_j\right)\right\}_{j \in [m]}, H-h+1, \alpha, \epsilon, \delta'\right)
$
\ENDFOR
\STATE [S] Compute $\bar Q_h, \hat Q_h, \hat \pi_h, \hat V_h$ as in \pref{eq:ucb vi 1}-\pref{eq:ucb vi 4}.
\ENDFOR
\ENDIF
\STATE [S] $\operatorname{SYNCHRONIZE} \gets \operatorname{FALSE}$
\FOR{$j \in \mathcal{G}$}
\STATE[A] $\operatorname{Sync}_j \gets \operatorname{FALSE}$, \ Sample   
$\left\{(s_{h}^{j,k}, a_{h}^{j,k}, r_{h}^{j,k}, s_{h+1}^{j,k})\right\}_{h\in[H]}$
under $\left\{\hat \pi_h\right\}_{h=1}^H$
\STATE [A] $\forall h$, $N_h^j(s_{h}^{j,k}, a_{h}^{j,k}) \gets N_h^j(s_{h}^{j,k}, a_{h}^{j,k}) + 1,\quad D_{h}^j \gets D_h^j \cup \left\{(s_{h}^{j,k}, a_{h}^{j,k}, r_{h}^{j,k}, s_{h+1}^{j,k})\right\}$
\STATE [A] Send Sync request to Server, if $\operatorname{Sync}_j  \gets \one\left\{\max_{s,a,h}\frac{N_h^j(s,a)}{N_{h,j}^{\operatorname{old}}(s,a)} \ge 2\right\}$ is $\operatorname{TRUE}$.
\ENDFOR
\ENDFOR
\RETURN $\left\{\hat \pi_h\right\}_{h=1}^H$ \label{ln:agent sync}
\end{algorithmic}
\end{algorithm}

We are now ready to present the following regret bound for \textsc{Byzan-UCBVI}.
\begin{theorem}[Regret bound]\label{thm:byzan ucb vi}
Under \pref{def:byzan online}, if $\alpha \le \frac{1}{3}\left(1-\frac{1}{m}\right)$, for all $\delta < \frac{1}{4}$, with probability at least $1-3\delta$,
the total regret of \pref{alg:UCBVI} is at most
\begin{equation}
\sum_{k=1}^K \sum_{j \in \mathcal{G}}\left(V_1^*(s_1) - V_1^{\hat\pi_k^j}(s_1)\right)
=
\tilde O\left(
(1 + \alpha \sqrt{m})H^2S\sqrt{AmK\log (1/\delta)}
\right)
\end{equation}
\end{theorem}
\begin{remark}[Understanding the regret bound]
In \pref{alg:UCBVI}, the good agents are using the same policy and thus for all $j \in \Gcal$, $\hat \pi_k^j = \hat \pi_k$, where $\hat \pi_k$ is the policy calculated by the server in $k$-th episode.
By utilizing the batch structure,
\pref{alg:UCBVI} achieves a regret sublinear in $K$, even under Byzantine attacks.
Our regret is only $O(\sqrt{mK}+\alpha {m}\sqrt{K}))$ compared to the $O(m\sqrt{K} + m\alpha^{1/4}K^{3/4})$ regret in \citep{jadbabaie2022byzantine}. 
When $\alpha\leq 1/\sqrt{m}$, the dominating term $\sqrt{mK}$ is optimal even in the clean setting \citep{azar2017minimax}.
\end{remark}
\begin{remark}[Communication cost]
Because each agent runs $K$ episodes in total, count of each of the $(s,a,h)$ tuples doubles at most 
$\lfloor\log_2 K\rfloor$
times during training. Thus each good agent will send at most
$SAH\lfloor\log_2 K\rfloor$ 
sync requests. 
The bad agents can only send logarithmic number of effective request because of the checking step in \pref{ln:sync check}.
As a result, there will be at most $mSAH\lfloor\log_2 K\rfloor$ synchronization episodes in total. 
The communication inside one synchronization episode includes the following:
at least one agent sends a sync request;
inside the value iteration, the server will send estimated value functions at $H$ steps to each agents; 
each of the good agents will send the estimated Bellman operator for each $(s,a)$ pairs at $H$ steps and the counts to server. 
Importantly, the agents only need to send summary statistics, instead of the raw dataset to server, this preserves the data privacy of individual agents \citep{sakuma2008privacy, liu2019privacy}. 
\end{remark}
\begin{remark}[Switching cost]
Switching cost measures the number of policy changes. 
Algorithms with low switching cost is favorable in real world applications \citep{bai2019provably, zhang2020almost, gao2021provably}.
\pref{alg:UCBVI} only performs policy updates in the synchronization episodes, its switching cost is thus at most $mSAH\lfloor\log_2 K\rfloor$.
\end{remark}

\section{Byzantine-Robust Offline RL \label{sec:pevi}}
In the offline setting, we assume the server has access to a set of data batches while some data batches are corrupted. 
The goal of the server is to find a nearly optimal policy without further interaction with the environment.
Specifically:
\begin{definition}[Distributed offline RL with Byzantine corruption]\label{def:byzan offline}
The server has access to an offline data set with $m$ data batches $\bigcup_{j \in [m]}D_j$, including $(1-\alpha) m$ good batches $\Gcal$ and $\alpha m$ bad batches $\Bcal$,
where $D_j := \bigcup_{h \in [H]}D_j^h := \bigcup_{h \in [H]} \left\{\left(s_h^{j,k}, a_h^{j,k}, r_h^{j,k},{s_h^{\prime}}^{j,
k}\right)\right\}_{k=1}^{K_j}$.
We make an assumption on the data generating process similar to \citep{wang2020statistical}.
Precisely, for all $j\in\Gcal$, $D_j$ is drawn from an unknown distribution $\left\{\nu^j_h\right\}_{h=1}^H$,
where for each $h \in [H]$, $\nu_h^j \in \Delta\left(\Scal \times \Acal\right)$.
For all $h,j,k$, $\left(s_h^{j,k}, a_h^{j,k}\right) \sim \nu_h^j$, ${s_h^{\prime}}^{j,k} \sim P_h(\cdot |s_h^{j,k}, a_h^{j,k})$
and $r_h^{j,k}$ is an instantiation of $R_h\left(s_h^{j,k}, a_h^{j,k}\right)$.
For any $j \in \Bcal$ (i.e. bad batches), $D_j$ can be arbitrary.
\end{definition}

The performance is measured by the suboptimality w.r.t. a deterministic comparator policy $\tilde \pi$ (not necessarily an optimal policy):
\begin{equation}\label{eq:sub-opt}
\subopt\left(\pi, \tilde \pi\right)  := V_1^{\tilde \pi}(s_1) - V_1^{\pi}(s_1).
\end{equation}
In the offline setting, the server cannot interact with the MDP. So our result relies heavily on the quality of the dataset.
As we will see in the analysis, the suboptimality gap \pref{eq:sub-opt} can be upper bounded by the estimation error of Bellman operator along the trajectory of $\tilde \pi$.
As a result, we do not need full coverage over the whole state-action space.
Instead, we only need the offline dataset to have proper coverage over $\{d_h^{\tilde\pi}\}_{h=1}^H$,
the state distribution of policy $\tilde \pi$ at each step $h$.
To characterize the data coverage, for any $s,a,h$, we define the counts on $(s,a,h)$ tuples by:
\begin{equation}
N_h^j(s,a) := \sum_{k \in [K_j]}\one\left\{(s_h^{j,k}, a_h^{j,k}) = (s, a)\right\}, \quad \forall j \in [m].
\end{equation}
When calling \pref{alg:int}, the large data batches might be clipped in \pref{ln:clip}.
By definition, the clipping threshold is bounded between:
$N_h^{\Gcal, \cut_1}(s,a)$,
the $(\alpha m+1)$-th largest of $\left\{N_h^j(s,a)\right\}_{j\in\Gcal}$
and $N_h^{\Gcal, \cut_2}(s,a)$, the $(2\alpha m+1)$-th largest of $\left\{N_h^j(s,a)\right\}_{j\in\Gcal}$.
We define three quantities 
$p^{\Gcal,0}, \kappa, \kappa_{\text{even}}$
to characterize the quality of the offline dataset.
The first quantity describes the density of $\tilde \pi$ trajectory that are not properly covered by the offline dataset:
\begin{definition}[Measure of insufficient coverage]\label{def:cover}
We define $p^{\Gcal,0}$ as the probability of $\tilde \pi$ visiting an $(s,h,a)$ tuple that is insufficiently covered by the logged data, namely
\begin{equation}
p^{\Gcal,0}
:=
\sum_{h=1}^H \EE_{d_h^{\tilde \pi}}
\left[\one\left\{
N_h^{\Gcal, \cut_2}(s,\tilde \pi(s)) =0
\right\}\right]
\end{equation}
\end{definition}
Recall that \pref{alg:int} requires there are at least $(2\alpha m + 1)$ non-empty data batches to make an informed decision. $p^{\Gcal,0}$ measures an upper bound on the total probability under $d^{\tilde \pi}$ to encounter an $(s,h,a)$ on which \textsc{Weighted-Clique} cannot return a good mean estimator.

We now introduce $\kappa$, the density ratio between the $d^{\tilde \pi}$ and the empirical distribution of the uncorrupted offline dataset.
$\kappa$ quantifies the portion of useful data in the whole dataset and is commonly used in the offline RL literature \citep{rashidinejad2021bridging, zhang2021corruption}.
We only focus on the $(s,a,h)$ tuples excluded by $p^{\Gcal,0}$ in \pref{def:cover}:
\begin{definition}[density ratio] \label{def:rel con num}
We use $\left\{\Ccal_h\right\}_{h=1}^H$ to denote the state space (in the support of $\left\{d_h^{\tilde \pi}\right\}_{h=1}^H$) 
that have proper clean agents coverage:
\begin{equation}
\Ccal_h = \left\{s | N_h^{\Gcal, \cut_2}(s,\tilde \pi(s)) >0\right\}    
\end{equation}

We use $\kappa$ to denote the density ratio 
between the state distribution of policy $\tilde \pi$ and the empirical distribution over the uncorrupted offline dataset:
\begin{equation}
\kappa := \max_{h\in[H]}\max_{s\in\Ccal_h} \frac{d_h^{\tilde \pi}(s)}
{\sum_{j \in \Gcal} N_h^j(s, \tilde \pi_h(s)) / \sum_{j\in\Gcal}K_j}
\end{equation}
\end{definition}
As we can see in \pref{thm:weighted int}, the accuracy of \pref{alg:int} heavily depend on the evenness of the batches. 
Even if there are some good batches with a large amount of data, those extra data are not useful (cf. \pref{rem:imposs lg bt}). 
We define the following quantity to measure the information loss in the clipping step (\pref{ln:clip} in \pref{alg:int}):
\begin{definition}[Unevenness of good agents coverage]\label{def:balance}
\begin{equation}
\kappa_{\text{even}} := \max_{h\in[H]}\max_{s\in\Ccal_h}
\frac{{\sum_{j \in \Gcal}N_h^j(s, \tilde \pi_h(s))}}{\sum_{j\in\Gcal} \tilde N_h^{j,\cut_2}(s, \tilde \pi_h(s))} 
\frac{m(1-\alpha) N_h^{\Gcal, \cut_1}(s,\tilde \pi_h(s))}{\sum_{j\in\Gcal} \tilde N_h^{j,\cut_2}(s, \tilde \pi_h(s))} 
\end{equation}
where 
$
\tilde N_h^{j,\cut_2}(s, \tilde \pi_h(s)) = \max\left(N_h^{\Gcal, \cut_2}(s,\tilde \pi_h(s)), N_h^j(s, \tilde \pi_h(s))\right)
$
\end{definition}
Intuitively, 
$\kappa_{\text{even}}$ describes the evenness of good agent coverage. It measures both
how many data in large batches are cut off by the clip step and
the unevenness of the batches after clipping.
We includes $N_h^{\Gcal, \cut_1}(s,\tilde \pi_h(s))$ and $N_h^{\Gcal, \cut_2}(s,\tilde \pi_h(s))$, instead of the true clipping threshold, meaning $\kappa_{\text{even}}$ serves as an upper bound of the actual unevenness resulting from running the algorithm.
For example, suppose $\alpha m > 1$:
if for any $s,a,h,j$, $N_h^{j}(s,a) = n$, then $\kappa_{\text{even}} = 1$;
if for any $s,a,h$, there is one good data batch with size $L m$ for some $L > 1$ while the others have size $1$,
then $N_h^{\Gcal, \cut_1}(s,a) = N_h^{\Gcal, \cut_2}(s,a) = 1$ and
$\kappa_{\text{even}} = \frac{Lm+(1-\alpha)m-1}{(1-\alpha)m}\frac{(1-\alpha)m}{(1-\alpha)m}\approx L+1$, meaning $\kappa_{\text{even}}$ increases as the batches become less even.

Remarkably, all three quantities defined above only depend on the $(s,a,h)$ counts of the good data batches.

Given the above setup, we now present our second algorithm, \textsc{Byzan-PEVI}, a Byzantine-Robust variant of pessimistic value iteration \citep{jin2021pessimism}. Similar to the online setting, we use our \textsc{Weighted-Clique} (without perturbation) algorithm to approximate the Bellman operator and use the estimation error to design PESSIMISTIC bonus for the value iteration.
\textsc{Byzan-PEVI} (\pref{alg:byzan PEVI}) runs pessimistic value iteration (\pref{eq:pevi 1}-\pref{eq:pevi 4}) and calls \textsc{Weighted-Clique} as a subroutine to robustly estimate the Bellman operator using offline dataset $D$:
\begin{align}
&\bar Q_h(\cdot, \cdot) = \left(\hat \Bell_h \hat V_{h+1}\right)(\cdot, \cdot) - \Gamma_h(\cdot,\cdot)\label{eq:pevi 1}
, \quad
\hat Q_h(\cdot, \cdot) = \min\left\{\bar Q_{h}(\cdot,\cdot), H-h+1\right\}^+\\
&\hat \pi_h(\cdot) = \argmax_{a} \hat Q_h(\cdot, a), \quad
\hat V_h(\cdot) = \max_{a} \hat Q_h(\cdot, a) \label{eq:pevi 4}
\end{align}

\begin{algorithm}
\caption{\textsc{Byzan-PEVI}} \label{alg:byzan PEVI}
\begin{algorithmic}[1]
\REQUIRE $D := \bigcup_{j\in[m]} D_j := \bigcup_{h \in [H]}D_j^h := \bigcup_{h \in [H]} \left\{\left(s_h^{j,k}, a_h^{j,k}, r_h^{j,k},{s_h^{\prime}}^{j,
k}\right)\right\}_{k=1}^{K_j}$, $\alpha$, $\delta$
\STATE $\delta' \gets \frac{\delta}{H|\Scal||\Acal|m}$
\STATE $\hat V_{H+1}(\cdot) \gets 0$
\FOR{$h = H, H-1, \ldots, 1$}
\STATE $\sigma \gets H-h+1$
\FOR{$(s,a) \in \mathcal{S} \times \mathcal{A}$}
\FOR{$j\in [m]$}
\STATE $n_j\gets \sum_{k \in [K_j]}\one\left\{(s_h^{j,k}, a_h^{j,k}) = (s, a)\right\}$
\STATE 
$x_j\gets \frac{1}{N_h^j(s,a)}\sum_{(s,a, r,s') \in D_h^j}\left(r + \hat V_{h+1}(s')\right)$
\ENDFOR
\IF{$|j \in [m]: n_j > 0| \ge 2\alpha m + 1$}
\STATE 
$
\left(\hat \Bell_h \hat V_{h+1}\right)(s,a), \Gamma_h(s,a) 
\gets 
\textsc{Weighted-Clique}\left(\left\{\left(
x_j,n_j\right)\right\}_{j=1}^m, \sigma, \alpha, \delta'\right)
$
\ELSE
\STATE $\left(\hat \Bell_h \hat V_{h+1}\right)(s,a) \gets 0$, $\Gamma_h(s,a) \gets H-h+1$
\ENDIF
\ENDFOR
\STATE Compute $\bar Q_h, \hat Q_h, \hat \pi_h, \hat V_h$ as in \pref{eq:pevi 1}-\pref{eq:pevi 4}.
\ENDFOR
\RETURN $\left\{\hat \pi_h\right\}_{h=1}^H$
\end{algorithmic}
\end{algorithm}

\begin{theorem}\label{thm:off-line VI}
Given any deterministic comparator policy $\tilde \pi$, under \pref{def:byzan offline}, \pref{def:cover}, \pref{def:rel con num} and \pref{def:balance}:
for any $\delta$, $\alpha < \frac{1}{3}$, with probability at least $1-\delta$, \pref{alg:byzan PEVI} outputs a policy $\hat \pi$ with:
\begin{align}
\subopt\left(\hat\pi, \tilde\pi\right) \le& 2H p^{\Gcal,0}
+ O\left(\sqrt{\kappa\kappa_{\text{even}}}H^2\sqrt{|\Scal|}
\frac{1+\sqrt{m}\alpha}{\sqrt{\sum_{j\in\Gcal}K_j}}
\sqrt{\log\frac{H|\Scal||\Acal|m}{\delta}}
\right).
\end{align}
\end{theorem}
\begin{remark}
Compared to \citep{zhang2021corruption}, there is no non-diminishing term in the bound. 
Meaning the suboptimality gap vanishes as the good agents collect more data.
To the best of our knowledge, this is the first result for Byzantine-robust offline RL.
\end{remark}
\begin{remark}[Offline v.s. online RL]
Our offline RL results are more involved and notation heavy due to the nature of the problem. In the offline RL setting, learner has no control over the data generating process, and each data source can be arbitrarily different. The agent can only passively rely on the robust mean estimator we designed and the pessimism principle to learn as well as the data permits. In contrast, in the online setting, the learner has complete control over the clean agents' data collection process. Our algorithm \textsc{Byzan-UCBVI} enables the server to realize its full potential and obtain a tighter and cleaner sample complexity guarantee.
\end{remark}

\section{Conclusion\label{sec:conclusion}}
To summarize, in this work, we first present \textsc{Weighted-Clique}, a robust mean estimation algorithm for learning from uneven batches that can be of independent interest. 
Building upon \textsc{Weighted-Clique}, we propose byzantine-robust online (\textsc{Byzan-UCBVI}) and the first byzantine-robust offline (\textsc{Byzan-PEVI}) reinforcement learning algorithms in distributed setting.
Several questions remain open:
(1) Can we provide a complete characterization of the information-theoretical lower bound for robust mean estimation from uneven batches?
(2) Can we extend our RL  algorithms to the function approximation setting?

\bibliography{reference}


\newpage
\appendix
\onecolumn
\section{More discussion on \pref{alg:int}:\textsc{Weighted-Clique} \label{sec:w-c dis}}
\subsection{Impossible result}
\begin{theorem}[impossibility result]\label{thm:impsb}
There exists a distribution $\Dcal$, s.t. given $m$ data batches
$\left\{\left\{x_j^i\right\}_{i=1}^{n_j}\right\}_{j\in[m]}$
generated under \pref{def:rb ln fm bt}, every robust mean estimation algorithm $\mathscr{A}$ suffers an error at least 
\begin{equation}
\Omega\left(\frac{1}{\sqrt{N}}\right) 
\end{equation}
even $\mathscr{A}$ knows some of the batches are clean,
where $N$ is the sum of sizes of the smallest $(1-2\alpha)m$ good batches.
\end{theorem}

\begin{proof}[Proof of \pref{thm:impsb}]
Let $\Dcal$ be Bernoulli distribution with parameter $\frac{1}{2}$.
W.l.o.g., assume $\Gcal = [(1-\alpha)m]$, $n_1\le \cdots \le n_{(1-\alpha)m}$ and
$\Bcal = \left\{(1-\alpha)m+1, \ldots, m\right\}$. 
We assume algorithm $\mathscr{A}$ knows $[(1-2\alpha)m]$ is a subset of the good batches.

Let $\eta = \frac{1}{2\sqrt{N}} = \frac{1}{2\sqrt{\sum_{j=1}^{(1-2\alpha)m}n_j}}$.
Let the bad batches $\Bcal$ be i.i.d. samples from $\Dcal'$, a Bernoulli distribution with parameter $\frac{1}{2} + \eta$.
By Theorem 4 of \citep{paninski2008coincidence, chan2014optimal}, no algorithm can distinguish if the batches 
$\left\{x_1^i\right\}_{i=1}^{n_1},\ldots, \left\{x_{(1-2\alpha)m}^i\right\}_{i=1}^{n_{(1-2\alpha)m}}$ 
are sampled from $\Dcal$ or $\Dcal'$.
I.e. no algorithm can distinguish if 
$\left\{(1-2\alpha)m+1, \ldots, (1-\alpha)m\right\}$
are good batches or 
$\Bcal$
are good batches.
This means, given $m$ data batches
$\left\{\left\{x_j^i\right\}_{i=1}^{n_j}\right\}_{j\in[m]}$, every robust mean estimation algorithm suffers an error at least $\Omega\left(\frac{1}{\sqrt{N}}\right)$.
\end{proof}

\subsection{Adaption to good batch perturbation and distributed learning \label{sec:adp pert dist}}
Compared to \pref{alg:int}, \pref{alg:int pert} enlarges the confidence interval by $\epsilon$ on both endpoints due to the perturbation and only requires some sufficient statistics from the batches, instead of the whole dataset.
When $\ncut >0$, meaning there are at least $2\alpha m + 1$ non-empty batches, \pref{alg:int pert} runs a modified \textsc{Weighted-Clique} algorithm to calculate the mean estimation and the error upper bound (with an additional $\epsilon$ as an adjustment for $\epsilon$-cover argument in the proof of \pref{thm:byzan ucb vi}).
When $\ncut =0$, \pref{alg:int pert} returns $0$ and a trivial error upper bound.
\begin{algorithm}
\caption{\textsc{Pert-Weighted-Clique} (\textsc{PWC})} \label{alg:int pert}
\begin{algorithmic}[1]
\REQUIRE $\left\{\left(\hat x_j, n_j\right)\right\}_{j \in [m]}$, $\sigma$, $\alpha$, $\epsilon$, $\delta > 0$
\STATE $\ncut \gets \mbox{$(2\alpha m+1)$-th largest value in $\{n_j\}_{j\in[m]}$}$  
\IF{$\ncut > 0$}
\FOR{$j = 1,2,\ldots,m$}
\STATE 
$\tilde n_{j} \gets \min(n_j, \ncut)$ 
\STATE \label{ln:CI pert} 
$I_j \gets \left[\hat x_j - \frac{\sigma}{\sqrt{\tilde n_j}}\sqrt{2\log \frac{2m}{\delta}} - \epsilon,\hat x_j + \frac{\sigma}{\sqrt{\tilde n_j}}\sqrt{2\log \frac{2m}{\delta}} + \epsilon\right]$
\ENDFOR
\STATE \label{ln:cliq pert} 
$U^* \gets \argmax_{U \mbox{ s.t. }  \emptyset\neq\bigcap_{j \in U}I_j }|U|$
\RETURN \label{ln:wted mean pert}
$
\hat x \gets \frac{1}{\sum_{j \in U^*}\tilde n_j}\sum_{j \in U^*} \tilde n_j \hat x_j, \quad \operatorname{Error} \gets \mbox{RHS of \pref{eq:pert bd 5eps}}+\epsilon
$
\label{ln:mod bn}
\ELSE
\IF{$\Dcal$ is bounded between $a,b$}
\RETURN 
$
\hat x \gets 0, \quad \operatorname{Error} \gets b-a
$
\ELSE
\RETURN 
$
\hat x \gets 0, \quad \operatorname{Error} \gets \infty
$
\ENDIF
\ENDIF
\end{algorithmic}
\end{algorithm}

\section{Proof of \prettyref{thm:weighted int}}
To prove \prettyref{thm:weighted int}, we show \pref{eq:err bound} holds under some concentration event while the event happens with high probability.
We consider a slight more general setting where there could be perturbations to even good batches:
\begin{definition}[Robust mean estimation from batches]\label{def:rb ln fm bt pert}
There are $m$ data providers indexed by: $\{1, 2, \ldots, m\} =: [m]$. 
Among these providers, we denote the indexes of uncorrupted providers by $\Gcal$
and the indexes of corrupted providers by $\Bcal$,
where $\Bcal \cup \Gcal = [m]$, $\Bcal \cap \Gcal = \emptyset$,  $\left|\Bcal\right| = \alpha m$. 
Any uncorrupted providers has access to \textbf{perturbed} samples from a sub-Gaussian distribution $\Dcal$ with mean $\mu$ and variance proxy $\sigma^2$ 
(i.e. $
\EE_{X\sim \mathcal{D}}[X] = \mu
$
and
$
\EE_{X \sim \mathcal{D}}\left[\exp\left(s \left(X - \mu\right)\right)\right] \le    \exp\left({\sigma^2 s^2}/{2}\right)
$,
$
\forall s \in \RR.
$).
\textbf{
For each $j \in \Gcal$, a data batch 
$\left\{\tilde x_j^i\right\}_{i=1}^{n_j}$ 
is drawn from $\Dcal$, 
while a perturbed version 
$\left\{ x_j^i\right\}_{i=1}^{n_j}$ 
is sent to the learner, where $n_j$ can be arbitrary and $\left|\hat x_j^i - \tilde x_j^i\right| \le \epsilon$ for some $\epsilon \ge 0$.}
For $j \in \Bcal$, 
$\left\{x_j^i\right\}_{i=1}^{n_j}$
can be arbitrary.
\end{definition}
One can easily recover \pref{def:rb ln fm bt} by letting $\epsilon = 0$.
We prove the
We first define the concentration event as following:
\begin{definition}[Concentration event]
For all $j \in \Gcal$, define the event that the empirical mean of clean batches is close to the population mean as:
\begin{equation}
\Ecal_j := \left\{\left|\hat x_j - \mu\right| \le \frac{\sigma}{\sqrt{\tilde n_j}}\sqrt{2\log \frac{2m}{\delta}} + \epsilon
\right\}
\end{equation}
Define the event that the weighted average of empirical means of clean batches is close to the population mean as:
\begin{equation}
\Ecal_{{wa}}  
:=
\left\{
\left|\frac{1}{\sum_{j \in \Gcal}\tilde n_j}\sum_{j \in \Gcal}\tilde n_j \hat x_j - \mu\right| \le 
\frac{\sigma}{\sqrt{\sum_{j \in \Gcal}\tilde n_{j}}}\sqrt{2\log\frac{2}{\delta}} + \epsilon   
\right\}
\end{equation}
Let $\Ecal_{{conc}}$ be the event that events above happens together:
\begin{equation}
\Ecal_{{conc}} := \Ecal_{wa} \cap \bigcap_{j\in\Gcal}\Ecal_j
\end{equation}
\end{definition}
We can show $\Ecal_{{conc}}$ happens with high probability using Hoeffding's inequality:
\begin{lemma}\label{lem:conc good event}
$\PP\left(\Ecal_{conc}\right) \ge 1- 2\delta$.
\end{lemma}
\begin{proof}
See proof in \pref{sec:pf conc}.
\end{proof}
Under event $\Ecal_{conc}$, we can give an upper bound on the estimation error:
\begin{lemma}\label{lem:bd in good event}
Under event $\Ecal_{conc}$, if $n^\cut > 0$, \pref{alg:int pert} outputs a $\hat x$ with
\begin{equation}\label{eq:pert bd 5eps}
\left|
\hat x-  \mu
\right| 
\le 
\frac{2}{\sqrt{\sum_{j \in [m]}\tilde n_j}}
\sigma\sqrt{2\log\frac{2}{\delta}} 
+ \frac{8\alpha m\sqrt{\ncut}}{\sum_{j \in [m]}\tilde n_j}
\sigma\sqrt{2\log\frac{2m}{\delta}} + 5\epsilon
\end{equation}
\end{lemma}
\begin{proof}
See proof in \pref{sec:pf bd w conc}.
\end{proof}
\begin{proof}[Proof of \pref{thm:weighted int}]
Consider $\epsilon = 0$, i.e. no perturbation involved. By \pref{lem:conc good event} and \pref{lem:bd in good event}, with probability at least $1-2\delta$,
\begin{equation}
\left|
\hat x-  \mu
\right| 
\le 
\frac{2}{\sqrt{\sum_{j \in [m]}\tilde n_j}}
\sigma\sqrt{2\log\frac{2}{\delta}} 
+ \frac{8\alpha m\sqrt{\ncut}}{\sum_{j \in [m]}\tilde n_j}
\sigma\sqrt{2\log\frac{2m}{\delta}}
\end{equation}
\end{proof}

\subsection{Proof of \pref{lem:conc good event}\label{sec:pf conc}}
To prove \pref{lem:conc good event},
\begin{enumerate}
\item we first show that the perturbation changes the empirical mean of batches by at most $\epsilon$;
\item we can show the concentration bound of empirical means and weighted means for the \textbf{unperturbed} samples;
\item we can conclude by using the two results above and triangular inequality.
\end{enumerate}

\paragraph{the probability of event $\bigcap_{j\in{\Gcal}}\Ecal_{j}$:}
For all $j \in \Gcal$, let $\bar x_j$ be the empirical mean of \textbf{unperturbed} samples in batch $j$:
\begin{equation}
\bar x_j := \frac{1}{n_j}\sum_{i=1}^{n_j} \tilde x_j^i    
\end{equation}
By triangular inequality:
\begin{align}\label{eq:pert mean}
\left|\bar x_j - \hat x_j\right| =  \left|\frac{1}{n_j}\sum_{i=1}^{n_j} (x_j^i - \tilde x_j^i)\right|  
\le \frac{1}{n_j}\sum_{i=1}^{n_j}\epsilon = \epsilon
\end{align}

Since $\mathcal{D}$ is sub-Gaussian distribution, we can show the concentration of unperturbed samples $\bar x_j$:
for all good batch $j \in \Gcal$, 
\begin{equation}
\PP\left(\left|\bar x_j - \mu\right| > t\right) \le 2\exp\left(-\frac{n_j t^2}{2\sigma^2}\right) \le 2\exp\left(-\frac{\tilde n_j t^2}{2\sigma^2}\right)
\end{equation}
By union bound, with probability at least $1-\delta$, $\forall j \in \Gcal$,
\begin{equation}
\left|\bar x_j - \mu\right| \le \frac{\sigma}{\sqrt{\tilde n_j}}\sqrt{2\log\frac{2\left|\Gcal\right|}{\delta}} \le \frac{\sigma}{\sqrt{\tilde n_j}}\sqrt{2\log\frac{2m}{\delta}}   
\end{equation}
By triangular inequality, with probability at least $1-\delta$, $\forall j \in \Gcal$, 
\begin{equation}
\left|\hat x_j - \mu\right|  
\le 
\left|\hat x_j - \bar x_j\right| + \left|\bar x_j - \mu\right|    
\le 
\frac{\sigma}{\sqrt{\tilde n_j}}\sqrt{2\log\frac{2m}{\delta}} + \epsilon
\end{equation}
I.e. $\PP\left(\bigcap_{j\in\Gcal}\Ecal_{j}\right) \ge 1-\delta$.

\paragraph{the probability of event $\Ecal_{wa}$:} 
We first show the weighted average of empirical mean of the \textbf{unperturbed} sample i.e., 
$\frac{1}{\sum_{j' \in \Gcal}\tilde n_{j'}}\sum_{j \in \Gcal}\tilde n_{j} \bar x_j$
is a sub-Gaussian random variable:
firstly, note that the mean of the weighted average is $\mu$, i.e.  $\EE\left[\frac{1}{\sum_{j' \in \Gcal}\tilde n_{j'}}\sum_{j \in \Gcal}\tilde n_{j} \bar x_j\right] = \mu$.
By definition, we know for good batch $j \in \Gcal$, 
$\tilde x_j^1, \ldots, \tilde x_j^{n_j}$ are i.i.d. sub-Gaussian random variable with mean $\mu$ and variance proxy $\sigma^2$, i.e.
\begin{equation}
\EE\left[\exp\left(s \left(\tilde x_j^i - \mu\right)\right)\right] \le    \exp\left(\frac{\sigma^2 s^2}{2}\right)
\quad \forall s \in \RR.
\end{equation}
Since $\bar x_j  = \frac{1}{n_j}\sum_{i=1}^{n_j}\tilde x_j^i$:
for all $s \in \RR$,
\begin{align}
&\EE\left[\exp\left(s \left(\frac{1}{\sum_{j' \in \Gcal}\tilde n_{j'}}\sum_{j \in \Gcal}\tilde n_j \bar x_j - \mu\right)\right)\right]  
=  \prod_{j \in \Gcal}  \EE\left[\exp\left(s \left(\frac{1}{\sum_{j' \in \Gcal}\tilde n_{j'}}\tilde n_j (\bar x_j - \mu)\right)\right)\right] \\
= & \prod_{j \in \Gcal} \prod_{i \in [n_j]}  \EE\left[\exp\left(\frac{s}{\sum_{j' \in \Gcal}\tilde n_{j'}}\frac{\tilde n_j}{n_j} (\tilde x_j^i - \mu)\right)\right] 
\le  \prod_{j \in \Gcal} \prod_{i \in [n_j]}  \exp\left(\frac{\sigma^2}{2}\left(\frac{s}{\sum_{j' \in \Gcal}\tilde n_{j'}}\frac{\tilde n_j}{n_j} \right)^2\right) \\
\le & \prod_{i \in \Gcal} \prod_{i \in [n_j]}  \exp\left(\frac{\sigma^2}{2}\left(\frac{s}{\sum_{j' \in \Gcal}\tilde n_{j'}} \right)^2\right) 
= \exp\left(\frac{s^2}{2}\left(\frac{\sigma}{\sqrt{\sum_{j' \in \Gcal}\tilde n_{j'}}}\right)^2 \right)
\end{align}
This means $\frac{1}{\sum_{j' \in \Gcal}\tilde n_{j'}}\sum_{j' \in \Gcal}\tilde n_{j'} \bar x_j$ is a sub-Gaussian random variable with variance proxy $\frac{\sigma^2}{\sum_{j' \in \Gcal}\tilde n_{j'}}$.
Thus $\forall t > 0$,
\begin{equation}
\PP\left(\left|\frac{1}{\sum_{j' \in \Gcal}\tilde n_{j'}}\sum_{j \in \Gcal}\tilde n_{j} \bar x_j - \mu\right| > t \right)   \le 2\exp\left(-\frac{\sum_{j' \in \Gcal}\tilde n_{j'}t^2}{2\sigma^2}\right)
\end{equation}
Thus with probability at least $1-\delta$:
\begin{equation}
\left|\frac{1}{\sum_{j' \in \Gcal}\tilde n_{j'}}\sum_{j \in \Gcal}\tilde n_{j} \bar x_j - \mu\right| \le \frac{\sigma}{\sqrt{\sum_{j' \in \Gcal}\tilde n_{j'}}}\sqrt{2\log\frac{2}{\delta}}    
\end{equation}
This means:
\begin{align}\label{eq:clean_bound purt}
&\left|\frac{1}{\sum_{j' \in \Gcal}\tilde n_{j'}}\sum_{j \in \Gcal}\tilde n_{j} \hat x_j - \mu\right|\\
\le & \left|\frac{1}{\sum_{j' \in \Gcal}\tilde n_{j'}}\sum_{j \in \Gcal}\tilde n_{j} \bar x_j - \mu\right| + 
\left|\frac{1}{\sum_{j' \in \Gcal}\tilde n_{j'}}\sum_{j \in \Gcal}\tilde n_{j} \bar x_j - \frac{1}{\sum_{j' \in \Gcal}\tilde n_{j'}}\sum_{j \in \Gcal}\tilde n_{j} \hat x_j\right|\\
\le &\frac{\sigma}{\sqrt{\sum_{j' \in \Gcal}\tilde n_{j'}}}\sqrt{2\log\frac{2}{\delta}} + \epsilon   
\end{align}
I.e. $\PP\left(\Ecal_{wa}\right) \ge 1-\delta$.

By union bound $\PP\left(\Ecal_{conc}\right) = \PP\left(\Ecal_{wa} \cap \bigcap_{j\in\Gcal}\Ecal_j\right)\ge 1- 2\delta$.

\subsection{Proof of \pref{lem:bd in good event}\label{sec:pf bd w conc}}
By \pref{lem:conc good event}, we know the weighted average of empirical mean of good batches is a proper estimation for the population mean. 
Compared to $\Gcal$, the $U^*$ returned in \pref{ln:cliq pert} in \pref{alg:int pert} may remove some good batches and include some bad batches. 
Even though, as long as we can show:
\begin{enumerate}
\item \pref{ln:cliq pert} will not remove too many good batches and will not include too many bad batches;
\item the bad batches included in $U^*$ will not be significant
\end{enumerate}
then we can show that the $\hat x$ returned in \pref{ln:wted mean pert} is a reasonable estimation for $\mu$.

\paragraph{The structure of $U^*$:}
$U^*$ is the largest subset of batches with confidence interval intersection. 
The confidence intervals of all the good batches intersect under event $\bigcap_{j\in\Gcal}\Ecal_{j}$, thus $U^*$ should at least as large as $\Gcal$, thus it is not possible to remove too many good batches.
Furthermore, we can also show that we will not lose too much information, meaning significantly reduce the total number of samples and thus later on, we can show that the statistical rate will not be affected by too much. 
We make these ideas precise below.

Under event $\bigcap_{j\in\Gcal}\Ecal_{j}$, 
\begin{equation}
\mu \in \bigcap_{j\in\Gcal}I_j,    
\end{equation}
where $I_j$ is the confidence interval defined in \pref{ln:CI pert}.
Thus $\bigcap_{j\in\Gcal}I_j \neq \emptyset$.

Because $U^*$ maximizes
\begin{equation}
\max_{U \mbox{ s.t. }  \emptyset\neq\bigcap_{j \in U}I_j }|U|,    
\end{equation}
we know $|U^*| \ge |\Gcal| = (1-\alpha)m$.
Furthermore, $U^*$ can include at most $\alpha m$ batches, this means
$U^*$ includes at least $(1-2\alpha m)$ good batches. Formally:
\begin{equation}
|U^* \cap \Gcal| = |U^* \setminus\Bcal| \ge |U^*| - |\Bcal| \ge (1-2\alpha )m.    
\end{equation}

Now we show $U^*$ is not losing too much information, i.e. 
$
\sum_{j\in U^*} \tilde n_j \ge  \frac{1}{2}\sum_{j \in [m]} \tilde n_j    
$.
By definition of $\ncut$, there are at least $2 \alpha m + 1$ batches in $[m]$ such that $\tilde n_j = \ncut$. 
Because $U^*$ removes at more $\alpha m$ batches, there are at least $\alpha m + 1$ batches in $U^*$ such that $\tilde n_j = \ncut$. I.e.
\begin{align}
\left|\left\{j \in U^* :\tilde n_j = \ncut\right\}\right| 
= &\left|\left\{j \in [m] :\tilde n_j = \ncut\right\}\right| 
- \left|\left\{j \in [m]\setminus U^* :\tilde n_j = \ncut\right\}\right|\\
\ge & \left|\left\{j \in [m] :\tilde n_j = \ncut\right\}\right| 
- \left|[m]\setminus U^*\right| \\
\ge & 2\alpha m + 1 - \alpha m = \alpha m + 1 \label{eq:ncut in U star}
\end{align}
This means the information loss $\sum_{j \in [m]\setminus \Gcal} \tilde n_j$ can be bounded by
$\sum_{j \in U^*} \tilde n_j$, formally:
\begin{align}
2\sum_{j\in U^*} \tilde n_j - \sum_{j \in [m] } \tilde n_j= & \sum_{j \in U^*}\tilde n_j + \sum_{j \in U^* }\tilde n_j - \sum_{j \in [m]\cap U^*}\tilde n_j - \sum_{j \in [m]\setminus U^*}\tilde n_j \\
= & \sum_{j \in U^*}\tilde n_j - \sum_{j \in [m]\setminus U^*}\tilde n_j \ge (\alpha m +1)\ncut - \alpha m \ncut \ge 0
\end{align}
Thus we have:
\begin{equation}\label{eq:sm info loss}
\sum_{j\in U^*} \tilde n_j \ge  \frac{1}{2}\sum_{j \in [m]} \tilde n_j.
\end{equation}

\paragraph{Bad batches in $U^*$:}
In order for a bad batch $i$ to survive in $U^*$, its confidence interval $I_i$ must intersect with each good batches' confidence interval in $U^*$.
In particular, $I_i$ must intersect with the good batch in $U^*$ with largest $\tilde n_j$. 
By definition, there are at least $\alpha m + 1$ good batches with $\tilde n_j = \ncut$. 
Because $U^*$ excludes at most $\alpha m$ good batches, there are at least one good batch (denote by $j^*$), s.t. $\tilde n_{j^*} = \ncut$.

Thus $\forall j \in U^* \cap \Bcal$, $I_i \cap I_{j^*} \neq \emptyset$.
Which means, there exists some point $x$, s.t. $x \in I_i \cap I_{j^*}$, thus
\begin{align}
&\left|\hat x_i - \hat x_{j^*}\right| \le 
\left|\hat x_i - x\right| + \left|x-\hat x_{j^*}\right|\\
\le &  \frac{\sigma}{\sqrt{\tilde n_i}}\sqrt{2\log \frac{2m}{\delta}} + \epsilon
+  \frac{\sigma}{\sqrt{\tilde n_{j^*}}}\sqrt{2\log \frac{2m}{\delta}} + \epsilon\\
=& \left(\frac{1}{\sqrt{\tilde n_i}} + \frac{1}{\sqrt{\ncut}}\right) \sigma \sqrt{2\log \frac{2m}{\delta}} 
+ 2\epsilon.
\end{align}
Furthermore, under event $\bigcap_{j\in\Gcal}\Ecal_{j}$, 
\begin{equation}
\left|\hat x_{j^*}-\mu\right| 
\le  \frac{\sigma}{\sqrt{\ncut}}\sqrt{2\log \frac{2m}{\delta}} + \epsilon
\end{equation}
By triangular inequality, $\hat x_i$ we not bee too far away from $\mu$:
\begin{equation}\label{eq:bd bc bd}
\left|\hat x_i - \mu\right| 
\le \left|\hat x_i - \hat x_{j^*}\right| + \left|\hat x_{j^*} - \mu\right|
= \left(\frac{1}{\sqrt{\tilde n_i}} + \frac{2}{\sqrt{\ncut}}\right) \sigma \sqrt{2\log \frac{2m}{\delta}} 
+ 3\epsilon
\end{equation}

\paragraph{Error decomposition:}
As mentioned earlier, we can decompose the estimation of $\hat x$ returned by \pref{alg:int pert} by: 
statistical error (with potential information loss), term $\mathscr{A}_1$ in \pref{eq:est err decom};
error coming from including bad batches, term $\mathscr{A}_2$ in \pref{eq:est err decom};
error coming from removing good batches, term $\mathscr{A}_3$ in \pref{eq:est err decom}.
Specifically:
\begin{align}
&\left|\hat x - \mu\right| 
=  \frac{1}{\sum_{j \in U^*}\tilde n_j}\left|\sum_{j \in U^*} \tilde n_j (\hat x_j  - \mu)\right|\\ 
= &\frac{1}{\sum_{j \in U^*}\tilde n_j}\left|\left(\sum_{j \in \Gcal} + \sum_{j \in U^*\cap \mathcal{B}} - \sum_{j \in \Gcal\setminus U^*}\right) \tilde n_j (\hat x_j  - \mu)\right| \\
\le & \frac{1}{\sum_{j \in U^*}\tilde n_j} 
\left(\left|\sum_{j \in \Gcal} \tilde n_j (\hat x_j  - \mu)\right|
+ \left|\sum_{j \in U^*\cap \mathcal{B}} \tilde n_j (\hat x_j  - \mu)\right|
+ \left|\sum_{j \in \Gcal\setminus U^*} \tilde n_j (\hat x_j  - \mu)\right|\right) \\
&\mbox{(this is by triangular inequality)}\\    
=: & \mathscr{A}_1 + \mathscr{A}_2 + \mathscr{A}_3 \label{eq:est err decom}
\end{align}
We can bound the first term $\mathscr{A}_1$ by \pref{eq:sm info loss} under event $\Ecal_{wa}$:
\begin{align}
\mathscr{A}_1 = &
\frac{1}{\sum_{j \in U^*}\tilde n_j} 
\left|\sum_{j \in \Gcal} \tilde n_j (\hat x_j  - \mu)\right| 
= \frac{\sum_{j \in \Gcal}\tilde n_j}{\sum_{j \in U^*}\tilde n_j} 
\left|\frac{1}{\sum_{j \in \Gcal}\tilde n_j}\sum_{j \in \Gcal} \tilde n_j (\hat x_j  - \mu)\right| \\
= &\frac{\sum_{j \in \Gcal}\tilde n_j}{\sum_{j \in U^*}\tilde n_j} 
\left|\frac{1}{\sum_{j \in \Gcal}\tilde n_j}\sum_{j \in \Gcal} \tilde n_j \hat x_j  - \mu\right|\\
\le& \frac{\sum_{j \in \Gcal}\tilde n_j}{\sum_{j \in U^*}\tilde n_j} \left(\frac{\sigma}{\sqrt{\sum_{j \in \Gcal}\tilde n_{j}}}\sqrt{2\log\frac{2}{\delta}} + \epsilon  \right) \quad \left(\mbox{By event $\Ecal_{wa}$}\right) \\
=& \frac{\sqrt{\sum_{j \in \Gcal}\tilde n_j}}{\sum_{j \in U^*}\tilde n_j} {\sigma}\sqrt{2\log\frac{2}{\delta}} + \frac{\sum_{j \in \Gcal}\tilde n_j}{\sum_{j \in U^*}\tilde n_j}\epsilon \\
\le& 2\frac{\sqrt{\sum_{j \in \Gcal}\tilde n_j}}{\sum_{j \in [m]}\tilde n_j} {\sigma}\sqrt{2\log\frac{2}{\delta}} 
+ \frac{\sum_{j \in \Gcal}\tilde n_j}{\sum_{j \in U^*}\tilde n_j}\epsilon
\quad \left(\mbox{By \pref{eq:sm info loss}}\right) \\
\le& 2\frac{\sqrt{\sum_{j \in [m]}\tilde n_j}}{\sum_{j \in [m]}\tilde n_j} {\sigma}\sqrt{2\log\frac{2}{\delta}} 
+ \frac{\sum_{j \in \Gcal}\tilde n_j}{\sum_{j \in U^*}\tilde n_j}\epsilon 
\quad \left(\mbox{By $\Gcal \subseteq [m]$}\right) \\
= & \frac{2}{\sqrt{\sum_{j \in [m]}\tilde n_j}} {\sigma}\sqrt{2\log\frac{2}{\delta}} 
+ \frac{\sum_{j \in \Gcal}\tilde n_j}{\sum_{j \in U^*}\tilde n_j}\epsilon
\end{align}
By \pref{eq:bd bc bd}, we can bound the second term $\mathscr{A}_2$ by:
\begin{align}
\mathscr{A}_2 & 
= \frac{1}{\sum_{j \in U^*}\tilde n_j} 
\left|\sum_{j \in U^*\cap \mathcal{B}} \tilde n_j (\hat x_j  - \mu)\right|
\le \frac{1}{\sum_{j \in U^*}\tilde n_j} 
\sum_{j \in U^*\cap \mathcal{B}} \tilde n_j\left| \hat x_j  - \mu\right|
\quad \left(\mbox{By triangular ineq}\right)\\
\le &\frac{1}{\sum_{j \in U^*}\tilde n_j} 
\sum_{j \in U^*\cap \mathcal{B}} \tilde n_j
\left(\left(\frac{1}{\sqrt{\tilde n_j}} + \frac{2}{\sqrt{\ncut}}\right) \sigma \sqrt{2\log \frac{2m}{\delta}} 
+ 3\epsilon\right) 
\quad \left(\mbox{By \pref{eq:bd bc bd}}\right)\\
\le &\frac{1}{\sum_{j \in U^*}\tilde n_j} 
\sum_{j \in U^*\cap \mathcal{B}} 
\left({\sqrt{\tilde n_j}} + \frac{2\tilde n_j}{\sqrt{\ncut}}\right) \sigma \sqrt{2\log \frac{2m}{\delta}} 
+ \frac{\sum_{j \in U^*\cap \mathcal{B}} \tilde n_j}{\sum_{j \in U^*}\tilde n_j} 3\epsilon \\
\le &\frac{1}{\sum_{j \in U^*}\tilde n_j} 
\sum_{j \in U^*\cap \mathcal{B}} 
3{\sqrt{\ncut}} \sigma \sqrt{2\log \frac{2m}{\delta}} 
+ \frac{\sum_{j \in U^*\cap \mathcal{B}} \tilde n_j}{\sum_{j \in U^*}\tilde n_j} 3\epsilon 
\quad \left(\mbox{By $\tilde n_j \le \ncut$}\right)\\
\le &\frac{3\alpha m{\sqrt{\ncut}}}{\sum_{j \in U^*}\tilde n_j} 
 \sigma \sqrt{2\log \frac{2m}{\delta}} 
+ \frac{\sum_{j \in U^*\cap \mathcal{B}} \tilde n_j}{\sum_{j \in U^*}\tilde n_j} 3\epsilon 
\quad \left(\mbox{$U^*$ includes at most $\alpha m$ bad batches}\right)
\end{align}

We can bound the third term $\mathscr{A}_3$ by:
\begin{align}
\mathscr{A}_3 = & \frac{1}{\sum_{j \in U^*}\tilde n_j} 
\left|\sum_{j \in \Gcal\setminus U^*} \tilde n_j (\hat x_j  - \mu)\right|
\le \frac{1}{\sum_{j \in U^*}\tilde n_j} 
\sum_{j \in \Gcal\setminus U^*} \tilde n_j\left|\hat x_j  - \mu\right|
\quad \left(\mbox{By triangular ineq}\right)\\
\le & \frac{1}{\sum_{j \in U^*}\tilde n_j} 
\sum_{j \in \Gcal\setminus U^*} \tilde n_j
\left(\frac{\sigma}{\sqrt{\tilde n_j}}\sqrt{2\log \frac{2m}{\delta}} + \epsilon\right)
\quad \left(\mbox{By event $\bigcap_{j\in\Gcal}\Ecal_j$}\right)\\
= & \frac{1}{\sum_{j \in U^*}\tilde n_j} 
\sum_{j \in \Gcal\setminus U^*} {\sigma}{\sqrt{\tilde n_j}}\sqrt{2\log \frac{2m}{\delta}} 
+ \frac{\sum_{j \in \Gcal\setminus U^*}\tilde n_j}{\sum_{j \in U^*}\tilde n_j} \epsilon \\
\le & \frac{\alpha m {\sqrt{\ncut}}}{\sum_{j \in U^*}\tilde n_j} 
 {\sigma}\sqrt{2\log \frac{2m}{\delta}} 
+ \frac{\sum_{j \in \Gcal\setminus U^*}\tilde n_j}{\sum_{j \in U^*}\tilde n_j} \epsilon\\
& \left(\mbox{Because $U^*$ excludes at most $\alpha m$ good batches and $\tilde n_j \le \ncut$}\right)
\end{align}
Note that the above upper bounds for $\mathscr{A}_2$ and $\mathscr{A}_3$ are still valid even if some of the $\tilde n_j$'s are zero. 

In conclusion, we can bound the estimation error by:
\begin{align}
\left|\hat x - \mu\right| \le & \mathscr{A}_1 + \mathscr{A}_2 + \mathscr{A}_3   \\
\le & \left(\frac{2}{\sqrt{\sum_{j \in [m]}\tilde n_j}} {\sigma}\sqrt{2\log\frac{2}{\delta}} 
+ \frac{\sum_{j \in \Gcal}\tilde n_j}{\sum_{j \in U^*}\tilde n_j}\epsilon\right) \\
&+ \left(\frac{3\alpha m{\sqrt{\ncut}}}{\sum_{j \in U^*}\tilde n_j} 
 \sigma \sqrt{2\log \frac{2m}{\delta}} 
+ \frac{\sum_{j \in U^*\cap \mathcal{B}} \tilde n_j}{\sum_{j \in U^*}\tilde n_j} 3\epsilon \right)\\
&+ \left(\frac{\alpha m {\sqrt{\ncut}}}{\sum_{j \in U^*}\tilde n_j} 
 {\sigma}\sqrt{2\log \frac{2m}{\delta}} 
+ \frac{\sum_{j \in \Gcal\setminus U^*}\tilde n_j}{\sum_{j \in U^*}\tilde n_j} \epsilon\right)\\
= & \frac{2}{\sqrt{\sum_{j \in [m]}\tilde n_j}} {\sigma}\sqrt{2\log\frac{2}{\delta}}
+ \frac{4\alpha m{\sqrt{\ncut}}}{\sum_{j \in U^*}\tilde n_j} 
 \sigma \sqrt{2\log \frac{2m}{\delta}} \\
& + \frac{\left(\sum_{j \in \Gcal}+\sum_{j \in U^*\cap \Bcal}\right)\tilde n_j}{\sum_{j \in U^*}\tilde n_j}\epsilon
+ \frac{\sum_{j \in U^*\cap \mathcal{B}} \tilde n_j}{\sum_{j \in U^*}\tilde n_j} 2\epsilon
+ \frac{\sum_{j \in \Gcal\setminus U^*}\tilde n_j}{\sum_{j \in U^*}\tilde n_j} \epsilon\\
\le & \frac{2}{\sqrt{\sum_{j \in [m]}\tilde n_j}} {\sigma}\sqrt{2\log\frac{2}{\delta}}
+ \frac{4\alpha m{\sqrt{\ncut}}}{\sum_{j \in U^*}\tilde n_j} 
 \sigma \sqrt{2\log \frac{2m}{\delta}} \\
& + \frac{\sum_{j \in [m]}\tilde n_j}{\sum_{j \in U^*}\tilde n_j}\epsilon
+ \frac{\alpha m \ncut}{\sum_{j \in U^*}\tilde n_j} 2\epsilon
+ \frac{\alpha m \ncut}{\sum_{j \in U^*}\tilde n_j} \epsilon \\
&\left(\mbox{By $\Gcal \cup \left(U^*\cap\Bcal\right) \subseteq [m]$, $|U^*\cap \Bcal|\le \alpha m$, $|\Gcal \setminus U^*|\le \alpha m$}\right)\\
\le & \frac{2}{\sqrt{\sum_{j \in [m]}\tilde n_j}} {\sigma}\sqrt{2\log\frac{2}{\delta}}
+ \frac{8\alpha m{\sqrt{\ncut}}}{\sum_{j \in [m]}\tilde n_j} 
 \sigma \sqrt{2\log \frac{2m}{\delta}} + 5\epsilon
\\
& \left(\mbox{By \pref{eq:ncut in U star} and \pref{eq:sm info loss}}\right)
\end{align}



\section{Proof of \pref{thm:byzan ucb vi}}
By following standard regret decomposition for UCB type of algorithm (see \citep{jin2020provably}), under the event that the estimation error of Bellman operator is bounded by bonus terms, we can decompose the regret by: 
\begin{enumerate}
\item the cumulative bonus term occurred in the trajectories of each good agent
\item a term that can be easier bounded by Azuma-Hoeffding's inequalities.
\end{enumerate}
By \pref{lem:bd in good event} and replacing \pref{lem:conc good event} with a variant for martingale, we can show the event mentioned above happens with high probability.
Unlike standard regret bound for tabular setting, we cannot directly use the telescope series to estimate the cumulative bonuses.
Instead, we first need to show that because each good agent is using the same policy in every episode, their trajectories have a lot of overlaps, meaning the $(s,a,h)$ counts of all good agents do not differ by too much.
Given that, we can simply the bound in \pref{lem:bd in good event} and use the telescope series.

We start by restating \pref{thm:byzan ucb vi}:
\begin{theorem}[Regret bound, \pref{thm:byzan ucb vi}]
If $\alpha \le \frac{1}{3}\left(1-\frac{1}{m}\right)$, for all $\delta<\frac{1}{4}$, with probability at least $1-3\delta$:
\begin{equation}
\sum_{k=1}^K \sum_{j \in \mathcal{G}}\left(V_1^*(s_1) - V_1^{\hat\pi^k}(s_1)\right)
=
\tilde O\left(
(1 + \alpha \sqrt{m})SH^2\sqrt{AKm\log \frac{1}{\delta}}
\right)
\end{equation}
\end{theorem}
We first give the high level idea of our proof:
\begin{enumerate}
\item We give an analysis under the intersection of three ``good events'':
\begin{itemize}
\item event $\mathcal{E}$: the estimation error of Bellman operator is upper-bounded by bonus (See \pref{sec:valid bn eve}, \pref{lem:valid bn event});
\item event $\mathcal{E}_{\mbox{even}}$: if the total count 
$\sum_{j\in\mathcal{G}}N_h^{j,k}(s,a,h)$ on some $(s,a,h)$ is large, then the counts of each agent differ by at most $2$ times (See \pref{sec:even agent eve}, \pref{lem:evenness of good agent});
\item event $\mathcal{E}_{\mbox{Azmua}}$: an error term in the regret decomposition is bounded by Azmua-Hoeffding bound.
\end{itemize}
\item Under event $\mathcal{E}$, we can decompose the regret into two terms (see \pref{sec:regret decom}, \pref{lem:regret decomp}): 
\begin{itemize}
\item a martingale with bounded difference which is controlled by Hoeffding bound under event $\mathcal{E}_{\mbox{Azmua}}$;
\item the cumulative bonus term, which can be bounded by telescoping series under event $\mathcal{E}_{\mbox{even}}$.
\end{itemize}
\end{enumerate}

We use $\bar Q_h^k$, $\hat Q_h^k$, $\hat \pi_h^k$, $\hat V_h^k$, $\hat \Bell_h^k$, $\Gamma_h^k$ to denote the variables used in the $k$-th episode. 
When a synchronization happens in episode $k$, those variables are the updated ones after the synchronization;
when there is no synchronization in episode $k$, those variables remain the same as last episode.
Let $N_h^{j,k}(s,a)$ be the counts on $(s,a,h)$ tuples in episode $k$ after the counts update.
Formally, 
We start by restating the data collection process and counts on $(s,a,h)$ tuples of each good agent $j \in \mathcal{G}$:
during the data collection process,
we allow all of the agents collect data together.
In the $k$-th episode, agent $j$ collects a multi-set of transition tuples using policy $d^{\hat \pi^k}$: $\left\{\left(s_{h}^{j,k}, a_{h}^{j,k}, r_{h}^{j,k}, s_{h+1}^{j,k}\right)\right\}_{h \in [H]}$.
\begin{equation}
D_{j,k} := \bigcup_{h \in [H]}D_{j,k}^h
:=\bigcup_{h \in [H]}\bigcup_{k' \le k}\left\{\left(s_{h}^{j,k'}, a_{h}^{j,k'}, r_{h}^{j,k'}, s_{h+1}^{j,k'}\right)\right\}
\end{equation}
$N_h^{j,k}(s,a)$ is given by:
\begin{equation}
N_h^{j,k}(s,a) = \sum_{h = 1}^H\sum_{(\tilde s, \tilde a, \tilde r,\tilde s') \in D_{j,k}^h}
\mathbf{1}\left\{(s,a) = (\tilde s, \tilde a)\right\}    
\end{equation}
We give the formal definition of good events below:
\begin{definition}
\begin{align}
\mathcal{E}_{\mbox{Azmua}} := &
\left\{\sum_{k=1}^K \sum_{j \in \mathcal{G}}\sum_{h=1}^H\left(\EE_{s'\sim P_h(\cdot \mid s_h^{j,k},a_h^{j,k})}\left[\hat V_{h+1}^k (s') - V_{h+1}^{\hat \pi^k}(s')\right] 
\right.\right.\\
&
\left.\left.
- \left(\hat V_{h+1}^k (s_{h+1}^{j,k}) - V_{h+1}^{\hat\pi^k}(s_{h+1}^{j,k})\right)\right) \le \sqrt{8mKH^3\log\frac{2}{\delta}}    
\right\}
\end{align}
\begin{equation}
\mathcal{E} := \left\{
\bigcap_{(s,a,h,k,f) \in \mathcal{S} \times \mathcal{A} \times H \times K \times [0,1]^{\mathcal{S}}} \left\{\left| \left(\hat \Bell_h^k f\right)(s,a) - \left(\Bell_h f\right)(s,a) \right| \le \Gamma_h^k(s,a)\right\}
\right\}
\end{equation}
For any $(s,a,h,k) \in \mathcal{S}\times\mathcal{A}\times [H] \times [K]$,
we define the follow event:
\begin{equation}
\mathcal{E}_{\mbox{even}}(s,a, h, k) := \left\{\mbox{if $\sum_{j\in \mathcal{G}} N_h^{j,k}(s,a)
\ge 400 m\log \frac{2 mKSAH}{\delta} 
$, then $\max_{i,j \in \mathcal{G}}\frac{N_h^{j,k}(s,a)}{N_h^{i,k}(s,a)} \le 2$}\right\}    
\end{equation}
And define 
\begin{equation}
\mathcal{E}_{\mbox{even}} := \bigcap_{s,a,h,K}  \mathcal{E}_{\mbox{even}}(s,a, h, k).   
\end{equation}
\end{definition}

\begin{proof}[Proof of \pref{thm:byzan ucb vi}]
By Azuma-Hoeffding inequality:
\begin{align}
&\PP\left(\overline{\mathcal{E}_{\mbox{Azmua}}}\right)  \le \delta
\end{align}
Then by union bound: \prettyref{lem:valid bn event} and \prettyref{lem:evenness of good agent} together implies for all $0 < \delta < \frac{1}{4}$:
\begin{align}
\PP\left(\overline{\mathcal{E}} \cup \overline{\mathcal{E}_{\mbox{even}}} \cup \overline{\mathcal{E}_{\mbox{Azmua}}}\right)
\le \PP\left(\overline{\mathcal{E}}\right) + \PP\left(\overline{\mathcal{E}_{\mbox{even}}}\right) + \PP\left(\overline{\mathcal{E}_{\mbox{Azmua}}}\right)
\le 3\delta
\end{align}
which means $\mathcal{E} \cap \mathcal{E}_{\mbox{even}} \cap \mathcal{E}_{\mbox{Azmua}}$ happens with probability at least $1- 3\delta$.

We now upper bound the regret under event $\mathcal{E} \cap \mathcal{E}_{\mbox{even}}\cap \mathcal{E}_{\mbox{Azmua}}$.
By \prettyref{lem:regret decomp} we can decompose the regret by:
\begin{align}
&\sum_{k=1}^K \sum_{j \in \mathcal{G}}\left(V_1^*(s_1) - V_1^{\hat\pi^k}(s_1)\right) \\
\le & 2\sum_{k=1}^K \sum_{j \in \mathcal{G}}\sum_{h=1}^H\Gamma_h^k(s_h^{j,k},a_h^{j,k}) \\
&+ 
\sum_{k=1}^K \sum_{j \in \mathcal{G}}\sum_{h=1}^H\left(\EE_{s'\sim P_h(\cdot \mid s_h^k,a_h^k)}\left[\hat V_{h+1}^k (s') - V_{h+1}^{\hat \pi^k}(s')\right] - \left(\hat V_{h+1}^k (s_{h+1}^{j,k}) - V_{h+1}^{\hat\pi^k}(s_{h+1}^{j,k})\right)\right) \\
& \left(\mbox{Under event $\mathcal{E}$}\right)\\
\le & 2\sum_{k=1}^K \sum_{j \in \mathcal{G}}\sum_{h=1}^H\Gamma_h^k(s_h^{j,k},a_h^{j,k})
+ \sqrt{8mKH^3\log\frac{2}{\delta}}    \\
& \left(\mbox{Under event $\mathcal{E}_{\mbox{Azmua}}$}\right)
\end{align}
We only need to upper bound the cumulative bonus. 
Suppose the policy is updated at the beginning of
$k_0+1, k_1+1, k_2+1, \ldots, k_l+1$-th episodes, 
with the data collected in the first 
$k_0, k_1, k_2, \ldots, k_l$-th episodes,
with $k_1 =1$. 
To simplify the notation, we define $k_{0} = 0$, $k_{l+1} = K$.

For convenience, in the following, we use $N_h^k(s,a)$ to denote the total count on $(s,a,h)$ tuples up to episode $k$ over all good agents:
\begin{equation}
N_h^k(s,a) := \sum_{j \in \mathcal{G}}N_h^{j,k}(s,a),   
\end{equation}
where $N_h^0(s,a) = 0$.
We can rearrange the cumulative bonus by summing over $(s,a)$ pairs:
\begin{align}\label{eq:bonus sum}
\sum_{k=1}^K \sum_{j \in \mathcal{G}}
\sum_{h=1}^H\Gamma_h^k(s_h^{j,k},a_h^{j,k})
= 
\sum_{h=1}^H\sum_{(s,a) \in \mathcal{S} \times \mathcal{A}} 
\sum_{t = 1}^{l+1}\Gamma_h^{k_{t-1}+1}(s, a) 
\left(
N_h^{k_t}(s,a) - N_h^{k_{t-1}}(s,a)
\right)
\end{align}
When there are less than $(2\alpha m+1)$ agents have coverage on some $(s,a,h)$ tuple, the bonus term $\Gamma_h^k(s,a)$ is trivially set to be $H-h+1$. 
In the following, we show that under event 
$\mathcal{E}_{\mbox{even}}$,
in \prettyref{eq:bonus sum},
for each $(s,a,h)$ tuple, there are at most $2N_0$ bonus term such that $\Gamma_h(s,a) = H-h+1$, where
\begin{equation}
N_0 := 400m\log \frac{2 mKSAH}{\delta}.    
\end{equation}
For any $(s,a,h)$, let $l_0(s,a,h)$ be such that:
\begin{equation}
N_h^{k_{l_0(s,a,h)-1}}(s,a) < N_0 \le N_h^{k_{l_0(s,a,h)}}(s,a).    
\end{equation}
This means when running the policy update at episode $k_{l_0(s,a,h)} + 1$, the total counts for $(s,a,h)$, i.e. $N_h^{k_{l_0(s,a,h)}}(s,a)$, is larger than $N_0$.
For any $k \ge k_{l_0(s,a,h)}$, we have
\begin{equation}
\sum_{j\in\mathcal{G}}N_h^{j,k}(s,a)    
= N_h^{k}(s,a) \ge N_h^{k_{l_0(s,a,h)}}(s,a)
\ge N_0.    
\end{equation}
By definition of $\mathcal{E}_{\mbox{even}}$, for any $k \ge k_{l_0(s,a,h)}$
\begin{equation}\max_{i,j \in \mathcal{G}}\frac{N_h^{j,k}(s,a)}{N_h^{i,k}(s,a)} \le 2
\end{equation}
this means for any $k \ge k_{l_0(s,a,h)}$,
$N_h^{j,k}(s,a) > 0, \forall j \in \mathcal{\mathcal{G}}$, meaning  all of the good agents have coverage on $(s,a,h)$, this means there are at least $(1-\alpha)m \ge 2\alpha m +1$ agents have coverage, and thus:

\begin{itemize}
\item Trivial bonus can only happens at $k \le k_{l_0(s,a,h)}$, i.e.
\begin{equation}\label{eq:bad bonus}
\Gamma_h^k(s,a) = H-h+1 \mbox{ only if }  k \le k_{l_0(s,a,h)}.    
\end{equation}
Furthermore, because in the algorithm, the agents synchronize and update their policy when or before any $(s,a,h)$ counts for a good agent doubles. 
I.e.: for all $(s,a,h,j,i) \in \mathcal{S} \times \mathcal{A} \times [H] \times \mathcal{G} \times [l]$:
\begin{equation}\label{eq:at most double}
N_h^{k_{t}}(s,a) \le 
2N_h^{k_{t-1}}(s,a)     
\end{equation}
This means
\begin{equation}\label{eq:at most 2N0}
N_h^{k_{l_0(s,a,h)}}(s,a) \le 
2N_h^{k_{l_0(s,a,h)-1}}(s,a) < 2N_0.
\end{equation}
Thus for each $(s,a,h)$ tuple, there are at most $2N_0$ bonus term such that $\Gamma_h(s,a) = H-h+1$.
\item for any $k \ge k_{l_0(s,a,h)}+1$
\begin{align}
\Gamma_h^k(s,a) = & \frac{6}{SAHKm}+
 \frac{2(H-h+1)}{\sqrt{\sum_{j \in [m]}\tilde N_h^{j,k-1}(s,a)}}
\sqrt{2\log\frac{2(SAHKm)^{3S}}{\delta}} \\
&+ \frac{8\alpha m\sqrt{N_h^{\operatorname{cut},k-1}(s,a)}}{\sum_{j \in [m]}\tilde N_h^{j,k-1}(s,a)}
(H-h+1)\sqrt{2\log\frac{2m(SAHKm)^{3S}}{\delta}} 
\end{align}
Where $N_h^{\operatorname{cut},k-1}(s,a)$ is the $(2\alpha m +1)$-largest among $\left\{N_{h}^{j,k-1}(s,a)\right\}$ and 
\begin{equation}
\tilde N_h^{j,k-1}(s,a) = \max\left(N_h^{\operatorname{cut},k-1}(s,a), N_h^{j,k-1}(s,a)\right);    
\end{equation}
For any $k -1\ge k_{l_0(s,a,h)}$,
$\max_{i,j \in \mathcal{G}}\frac{N_h^{j,k-1}(s,a)}{N_h^{i,k-1}(s,a)} \le 2$
implies $\forall j$, 
$\tilde N_h^{j,k-1}(s,a) \ge \frac{1}{2}N_h^{j,k-1}(s,a)$ 
and 
$\tilde N_h^{j,k-1}(s,a) \ge \frac{1}{2}N_h^{\operatorname{cut},k-1}(s,a)$.

This means for any $k \ge k_{l_0(s,a,h)}+1$
\begin{equation}
\frac{1}{\sqrt{\sum_{j \in [m]}\tilde N_h^{j,k-1}(s,a)}} \le 
\frac{\sqrt{2}}{\sqrt{\sum_{j \in [m]}N_h^{j,k-1}(s,a)}} = \frac{\sqrt{2}}{\sqrt{N_h^{k-1}(s,a)}}
\end{equation}
\begin{align}
\frac{m\sqrt{N_h^{\operatorname{cut},k-1}(s,a)}}{\sum_{j \in [m]}\tilde N_h^{j,k-1}(s,a)}    
= &
\frac{\sqrt{m}\sqrt{\sum_{j\in[m]}N_h^{\operatorname{cut},k-1}(s,a)}}{\sum_{j \in [m]}\tilde N_h^{j,k-1}(s,a)}\\
\le& \frac{\sqrt{m}\sqrt{2\sum_{j\in[m]}\tilde N_h^{j,k-1}(s,a)}}{\sum_{j \in [m]}\tilde N_h^{j,k-1}(s,a)}
\le 
\frac{2\sqrt{m}}{\sqrt{N_h^{k-1}(s,a)}}
\end{align}
Thus
\begin{align}\label{eq:bd bn}
\Gamma_h^k(s,a)
\le & \frac{4
+ 16\sqrt{2} \alpha \sqrt{m}}{\sqrt{N_h^{k-1}(s,a)}}
H\sqrt{\log\frac{2m(SAHKm)^{3S}}{\delta}}
+ \frac{6}{SAHKm}
\end{align}
\end{itemize}

We are know ready to bound the cumulative bonus:
\begin{align}
&\sum_{k=1}^K \sum_{j \in \mathcal{G}}
\sum_{h=1}^H\Gamma_h^k(s_h^{j,k},a_h^{j,k})
= 
\sum_{h=1}^H\sum_{(s,a) \in \mathcal{S} \times \mathcal{A}} 
\sum_{t = 1}^{l+1}\Gamma_h^{k_{t-1}+1}(s, a) 
\left(
N_h^{k_t}(s,a) - N_h^{k_{t-1}}(s,a)
\right) \\
= & \sum_{h=1}^H\sum_{(s,a) \in \mathcal{S} \times \mathcal{A}} 
\left(
\sum_{t = 1}^{l_0(s,a,h)}\Gamma_h^{k_{t-1}+1}(s, a) 
\left(
N_h^{k_t}(s,a) - N_h^{k_{t-1}}(s,a)
\right)
\right.\\
&
\left.+
\sum_{t = l_0(s,a,h)+1}^{l+1}\Gamma_h^{k_{t-1}+1}(s, a) 
\left(
N_h^{k_t}(s,a) - N_h^{k_{t-1}}(s,a)
\right)
\right) \\
\le & \sum_{h=1}^H\sum_{(s,a) \in \mathcal{S} \times \mathcal{A}} 
\left(
\sum_{t = 1}^{l_0(s,a,h)}\Gamma_h^{k_{t-1}+1}(s, a) 
\left(
N_h^{k_t}(s,a) - N_h^{k_{t-1}}(s,a)
\right)
\right.\\
&
+\sum_{t = l_0(s,a,h)+1}^{l+1}\frac{4
+ 16\sqrt{2} \alpha \sqrt{m}}{\sqrt{N_h^{k-1}(s,a)}}
H\sqrt{\log\frac{2m(SAHKm)^{3S}}{\delta}}
\left(
N_h^{k_t}(s,a) - N_h^{k_{t-1}}(s,a)
\right)\\
&\left.+
\sum_{t = l_0(s,a,h)+1}^{l+1}\frac{6}{SAHKm}
\left(
N_h^{k_t}(s,a) - N_h^{k_{t-1}}(s,a)
\right)
\right) \\
&\left(\mbox{By \pref{eq:bd bn}}\right)\\
=:& \mathscr{A}_1 + \mathscr{A}_2 + \mathscr{A}_3.
\end{align}
By \prettyref{eq:bad bonus} and \prettyref{eq:at most 2N0},
\begin{equation}
\mathscr{A}_1 \le  SAH^2N_h^{k_{l_0(s,a,h)}}(s,a)\le  2SAH^2N_0.   
\end{equation}
Because $k_{l+1} = K$,
\begin{align}
\mathscr{A}_3 \le & \frac{6}{SAHKm}\sum_{h=1}^H \sum_{(s,a)\in\Scal\x\Acal} N_h^{K}(s,a) 
= \frac{6}{SA}
\end{align}
By \prettyref{eq:at most double},
\begin{align}
&\sum_{t = l_0(s,a,h)+1}^{l+1}\frac{N_h^{k_t}(s,a) - N_h^{k_{t-1}}(s,a)}{\sqrt{N_h^{k_{t-1}}(s,a)}} 
\le (\sqrt{2}+1)\sum_{t = l_0(s,a,h)+1}^{l+1}\frac{N_h^{k_t}(s,a) - N_h^{k_{t-1}}(s,a)}{\sqrt{N_h^{k_t}(s,a)} + \sqrt{N_h^{k_{t-1}}(s,a)}}\\
=& (\sqrt{2}+1)\sum_{t = l_0(s,a,h)+1}^{l+1}\left({\sqrt{N_h^{k_t}(s,a)} - \sqrt{N_h^{k_{t-1}}(s,a)}}\right) 
\le (\sqrt{2}+1)\sqrt{N_h^{K}(s,a)}
\end{align}
By Cauchy–Schwarz inequality,
\begin{align}
\sum_{(s,a) \in \mathcal{S} \times \mathcal{A}}   \sqrt{N_h^{K}(s,a)}
\le \sqrt{\sum_{(s,a) \in \mathcal{S} \times \mathcal{A}} 1\sum_{(s,a) \in \mathcal{S} \times \mathcal{A}} {N_h^{K}(s,a)}}
=\sqrt{SAKm}
\end{align}
Thus 
\begin{align}
\mathscr{A}_2 \le & 
(\sqrt{2}+1)(4+16\sqrt{2}\alpha \sqrt{m})H^2\sqrt{SAKm}  \sqrt{\log\frac{2m(SAHKm)^{3S}}{\delta}}  \\
=&O\left((1+\alpha\sqrt{m})H^2S\sqrt{AKm}\sqrt{\log\frac{SAHKm}{\delta}} \right)
\end{align}
Thus
\begin{align}
\mathscr{A}_1 + \mathscr{A}_2 + \mathscr{A}_3
\le & O\left((1+\alpha\sqrt{m})H^2S\sqrt{AKm}\sqrt{\log\frac{SAHKm}{\delta}} \right)\\
&+ O\left(SAH^2 m\log \frac{2 mKSAH}{\delta}\right)
\end{align}

In conclusion:
\begin{align}
\sum_{k=1}^K \sum_{j \in \mathcal{G}}\left(V_1^*(s_1) - V_1^{\hat\pi^k}(s_1)\right)
\le & 2\sum_{k=1}^K \sum_{j \in \mathcal{G}}\sum_{h=1}^H\Gamma_h^k(s_h^{j,k},a_h^{j,k})
+ \sqrt{8mKH^3\log\frac{2}{\delta}}\\  
=&\tilde O\left(
(1 + \alpha \sqrt{m})SH^2\sqrt{AKm\log \frac{1}{\delta}}
\right)
\end{align}
\end{proof}

\subsection{the good event \texorpdfstring{$\mathcal{E}$}: \label{sec:valid bn eve}}
We first show that our bonus is a valid upper confidence bound for the estimated Bellman operator. 
Recall that our bonus term used in $k$-th episode 
is calculated based on the data collected in the first $k-1$-episodes. 
The bonus is given by:
\begin{itemize}
\item If $|j \in [m]: N_h^{j,k-1}(s,a) > 0| < 2\alpha m + 1$
\begin{equation}
\Gamma_h^k(s,a) = H - h + 1;    
\end{equation}
\item If $|j \in [m]: N_h^{j,k-1}(s,a) > 0| \ge 2\alpha m + 1$
\begin{align}
\Gamma_h^k(s,a) := & \frac{6}{SAHKm}+
\ \frac{2{(H-h+1)}}{\sqrt{\sum_{j \in [m]}\tilde N_h^{j,k-1}(s,a)}}
\sqrt{2\log\frac{2(SAHKm)^{3S}}{\delta}} \\
&+ \frac{8\alpha m\sqrt{N_h^{\operatorname{cut},k-1}(s,a)}}{\sum_{j \in [m]}\tilde N_h^{j,k-1}(s,a)}
{(H-h+1)}\sqrt{2\log\frac{2m(SAHKm)^{3S}}{\delta}} 
\end{align}
Where $N_h^{\operatorname{cut},k-1}(s,a)$ is the $(2\alpha m +1)$-largest among $\left\{N_{h}^{j,k-1}(s,a)\right\}$ and 
\begin{equation}
\tilde N_h^{j,k-1}(s,a) = \max\left(N_h^{\operatorname{cut},k-1}(s,a), N_h^{j,k-1}(s,a)\right).    
\end{equation}
\end{itemize}

To be precise:
\begin{lemma}[Valid bonus]\label{lem:valid bn event}
Let $\mathcal{E}$ be the following event:
\begin{equation}
\mathcal{E} = \left\{
\bigcap_{(s,a,h,k,f) \in \mathcal{S} \times \mathcal{A} \times H \times K \times [0,1]^{\mathcal{S}}} \left\{\left| \left(\hat \Bell_h^k f\right)(s,a) - \left(\Bell_h f\right)(s,a) \right| \le \Gamma_h^k(s,a)\right\}
\right\}
\end{equation}
Then, we have 
\begin{equation}
\PP\left(\mathcal{E}\right) \ge 1 - \delta
\end{equation}
\end{lemma}

To show that $\Ecal$ is a high probability event, we seek to utilize the result of \pref{thm:weighted int}. 
Since there are two obstacles, we need to make some modifications:
\begin{enumerate}
\item Because the transition tuples are collected sequentially, they are no longer i.i.d., which means \pref{lem:conc good event} does not hold trivially. To resolve this, we use the concentration of martingale (see \pref{lem:seq dt conc});
\item Event $\Ecal$ shows the concentration property of $\hat \Bell$ holds uniformly for infinitely many $f$'s. 
Thus a direct union bound does not apply.
Instead we need to use a cover number argument for all possible $f$'s,
which is standard (see \citep{jin2020provably}).
\end{enumerate}

\begin{proof}[Proof of \prettyref{lem:valid bn event}]
Let $\mathcal{E}'$ be the following event:
\begin{equation}\label{eq:eve cover}
\mathcal{E}' = \left\{N_h^{\operatorname{cut},k-1}(s,a) > 0\right\}.
\end{equation}
In the following, we decompose $\mathcal{E}$ by:
\begin{equation}
\mathcal{E} = \left(\mathcal{E} \cap \overline{\mathcal{E}'}\right) \cup \left(\mathcal{E} \cap \mathcal{E}'\right)
\end{equation}
and bound $\PP\left(\mathcal{E}\right)$ by law of total probability.

If $N_h^{\operatorname{cut},k-1}(s,a)=0$, 
because $\left(\hat \Bell_h^k f\right)(s,a) = 0$ and 
$\left(\Bell_h f\right)(s,a) \le H - h + 1$,
with probability $1$, $\forall (s,a,h,k,f) \in \mathcal{S} \times \mathcal{A} \times H \times K \times [0,1]^{\mathcal{S}}$,
\begin{equation}
\left| \left(\hat \Bell_h^k f\right)(s,a) - \left(\Bell_h f\right)(s,a) \right| \le \Gamma_h^k(s,a)    
\end{equation}
This means 
\begin{equation}
\PP\left(\mathcal{E} \cap \overline{\mathcal{E}'}\right) 
= \PP\left(\mathcal{E} | \overline{\mathcal{E}'}\right)  \PP\left(\overline{\mathcal{E}'}\right) = \PP\left(\overline{\mathcal{E}'}\right)
\end{equation}

If $N_h^{\operatorname{cut},k-1}(s,a) > 0$,
we use a covering number argument and union bound to bound the probability of event $\mathcal{E}$.

Consider 
$\mathcal{V}_{\epsilon} := \left\{
\frac{1}{\lceil 1/ \epsilon \rceil}, 
\frac{2}{\lceil 1/ \epsilon \rceil}, \ldots, 
\frac{H\lceil 1/ \epsilon \rceil}{\lceil 1/ \epsilon \rceil}
\right\}^{\mathcal{S}}$,
an $\epsilon$ cover of $[0,H]^{\mathcal{S}}$, in the sense of $\infty$-norm.
We can bound the cover number by
$\left|\mathcal{V}_{\epsilon}\right| \le \left(H\left(\frac{1}{\epsilon} + 1\right)\right)^{S}$.
This means $\forall f \in [0,H]^{\mathcal{S}}$, we can find an $V_f \in \mathcal{V}_\epsilon$, s.t. 
$\|f - V_f\|_\infty := \max_{x\in\mathcal{S}}|f(x) - V_f(x)| \le \epsilon$.
In another word, 
\begin{equation}
[0,H]^{\mathcal{S}} = \bigcup_{f_\epsilon \in \mathcal{V}_\epsilon}
\left\{f : \|f - f_\epsilon\|_\infty \le \epsilon\right\}.
\end{equation}

Importantly, unlike model based method without bad agents, our $\hat \Bell$ is not an linear operator, meaning we cannot trivially upper bound 
$\left| \left(\hat \Bell_h^k f\right)(s,a) - \left(\hat \Bell_h^k V_f\right)(s,a) \right| $ in the cover number argument. Instead, we need to use the continuity of error bound of our robust mean estimation \prettyref{alg:int pert}, meaning as long as each data point collected by each agent is not perturbed too much, then the estimation error bound does not increase too much.

Recall that in \prettyref{alg:UCBVI}, 
at episode $k$, if the agents decide to synchronize, then at each step $h$, given any function $f$, the clean agents will calculate empirical mean for 
\begin{equation}\label{eq:pert bell set}
\left\{r + f (s'): (s,a,r,s') \in D_h^{j,k}\right\}.
\end{equation}
Let $f_\epsilon$ be an element in $\mathcal{V}_{\epsilon}$, s.t. 
$\|f_\epsilon- f\|_\infty \le \epsilon$, this means set \prettyref{eq:pert bell set} is a perturbed version (by at most $\epsilon$) of 
\begin{equation}
\left\{r + f_\epsilon (s'): (s,a,r,s') \in D_h^{j,k}\right\}.
\end{equation}
This means given an $f_\epsilon \in \mathcal{V}_\epsilon$, for any $f$, s.t. 
$\|f - f_\epsilon\|_\infty \le \epsilon$,
\prettyref{alg:int pert} can be used to robustly estimate $\left(\Bell_h f_\epsilon\right)(s,a)$, given set \prettyref{eq:pert bell set}.
Furthermore, choosing $\epsilon = \frac{1}{SAHKm}$,
by \prettyref{lem:seq dt conc}, \prettyref{lem:total seq dt conc} and \pref{lem:bd in good event}, given any $s,a,h,k, f_\epsilon$, and 
any $f$, s.t. 
$\|f - f_\epsilon\|_\infty \le \epsilon$,
with probability at least
$1-\frac{\delta}{(SAHKm)^{3S}/(2mK)}$,
\begin{equation}
\left| \left(\hat \Bell_h^k f\right)(s,a) - \left( \Bell_h f_\epsilon\right)(s,a) \right| \le \Gamma_h^k(s,a) - \frac{1}{SAHKm}. 
\end{equation}

We can bound the $\left| \left(\hat \Bell_h^k f\right)(s,a) - \left(\Bell_h f\right)(s,a) \right|$ by:
\begin{align}
\left| \left(\hat \Bell_h^k f\right)(s,a) - \left(\Bell_h f\right)(s,a) \right| \le 
&\left| \left(\hat \Bell_h^k f\right)(s,a) - \left(\Bell_h f_\epsilon\right)(s,a) \right|+ \left| \left(\Bell_h f_\epsilon\right)(s,a) - \left(\Bell_h f\right)(s,a) \right| \\
\le & \left| \left(\hat \Bell_h^k f\right)(s,a) - \left(\Bell_h f_\epsilon\right)(s,a) \right| + \frac{1}{SAHKm}
\end{align}

Then 
\begin{align}
&\PP\left(\bigcup_{s,a,h,k,f}
\left\{\left| \left(\hat \Bell_h^k f\right)(s,a) - \left(\Bell_h f\right)(s,a) \right| > \Gamma_h^k(s,a)\right\}
\right) \\
\le &
\sum_{s,a,h,k}
\PP\left(\bigcup_{f\in[0,H]^{\mathcal{S}}}
\left\{\left| \left(\hat \Bell_h^k f\right)(s,a) - \left(\Bell_h f\right)(s,a) \right| > \Gamma_h^k(s,a)\right\}
\right) \\
\le &
\sum_{s,a,h,k}
\PP\left(\bigcup_{f_\epsilon \in \mathcal{V}_\epsilon}\bigcup_{f: \|f - f_\epsilon\|_\infty \le \epsilon}
\left\{\left| \left(\hat \Bell_h^k f\right)(s,a) - \left(\Bell_h f_\epsilon\right)(s,a) \right| + \frac{1}{SAHKm}  > \Gamma_h^k(s,a)\right\}
\right)\\
\le &
\sum_{s,a,h,k}\sum_{f_\epsilon \in \mathcal{V}_\epsilon}
\PP\left(\bigcup_{f: \|f - f_\epsilon\|_\infty \le \epsilon}
\left\{\left| \left(\hat \Bell_h^k f\right)(s,a) - \left(\Bell_h f_\epsilon\right)(s,a) \right| + \frac{1}{SAHKm}  > \Gamma_h^k(s,a)\right\}
\right)\\
\le &SAHK 
(H(1+HSAKm))^{S} 
\frac{\delta}{(SAHKm)^{3S}/(2mK)} 
\le  \delta
\end{align}
This means 
\begin{equation}
\PP\left(\mathcal{E} \cap {\mathcal{E}'}\right) 
= \PP\left(\mathcal{E} | {\mathcal{E}'}\right)  \PP\left({\mathcal{E}'}\right) 
\ge \left(1 - \delta\right)\PP\left({\mathcal{E}'}\right)
\ge \PP\left({\mathcal{E}'}\right) - \delta
\end{equation}
In conclusion,
\begin{equation}
\PP\left(\mathcal{E}\right) =   \PP\left(\mathcal{E} \cap {\mathcal{E}'}\right) + \PP\left(\mathcal{E} \cap \overline{\mathcal{E}'}\right)  
\ge \PP\left({\mathcal{E}}\right) +  \PP\left({\mathcal{E}'}\right) - \delta
= 1 - \delta.
\end{equation}
\end{proof}

\subsubsection{Concentration of estimation from good agents}
\begin{lemma}\label{lem:seq dt conc}
Let:
\begin{equation}
\left(\hat \Bell_h^{j,k} f \right)(s,a)
:=
\frac{1}{N_{h}^{j,k}(s,a)}\sum_{(s,a,r,s') \in D_h^{j,k}}r+f(s'),
\end{equation}
where we define $\frac{0}{0}=0$.
For any $f:\Scal \mapsto [H]$, 
and for any $(s,a,h,k)\in\Scal \x \Acal \x [H] \x [K]$
with probability at least $1-\delta/2$, $\Ecal_{conc-seq}(s,a,h,k)$ happens, where
\begin{equation}
\Ecal_{conc-seq}(s,a,h,k) = \bigcap_{j\in\Gcal}\Ecal_{c-seq}(s,a,h,j,k),
\end{equation}
and
\begin{align}
\Ecal_{c-seq}(s,a,h,j,k) := 
\left\{
\left|\left(\hat \Bell_h^{j,k} f \right)(s,a)  
-
\left(\Bell_h f \right)(s,a)  
\right|
\le 
\frac{H-h+1}{\sqrt{\tilde N_{h}^{j,k}(s,a)}}
\sqrt{2\log \frac{4Km}{\delta}} 
\right\}
\end{align}
\end{lemma}
\begin{proof}[Proof of \pref{lem:seq dt conc}]
We use the martingale stopping time argument in Lemma 4.3 of \citep{jin2018q}.

For each fixed $(s,a,h,j) \in \Scal \x \Acal \x [H] \x \Gcal$:
for all $t \in [K]$, define
\begin{equation}
\Fcal_t := \sigma\left(\bigcup_{t'\le t}\bigcup_{j \in [m]}\left\{\left(s_{h}^{j,t'},a_{h}^{j,t'},r_{h}^{j,t'},s_{h+1}^{j,t'}\right)\right\}_{h=1}^H\right).    
\end{equation}
Let 
\begin{equation}
X_t =  \sum_{(s,a,r,s') \in D_h^{j,t}}
\left(
r+f(s') -   \left(\Bell_h f \right)(s,a)  
\right)
\end{equation}
Then  $\left\{\left(\Fcal_t, X_t\right)\right\}_{t=1}^K$ is a martingale.
One observation is $X_{t_1} = X_{t_2}$ if agent $j$ did not visit $(s,a,h)$ in $t_1+1, t_1 + 2, \dots, t_2$-th episodes.
Thus we can use the stopping time idea to shorten the martingale sequence.

Define the following sequence of $t_i$'s: $t_0:=0$,
\begin{equation}
t_i := \min\left(\left\{t' \in [K]: t' \ge t_{i-1} \mbox{ and } (s_h^{j,t'}, a_h^{j,t'}) = (s,a)\right\}\cup \left\{K+1\right\}\right).
\end{equation}
Intuitively, $t_i$ is the episode when $(s,a, h)$ is visited by agent $j$ for the $i$-th time.
If agent $j$ visit $(s,a,h)$ for less than $i$ times, then $t_i = K+1$.
By definition, $t_i$ is a stopping time w.r.t. $\left\{\Fcal_t\right\}_{t=1}^K$.

By optional sampling theorem,
$\left\{\left(\Fcal_{t_i}, X_{t_i}\right)\right\}_{i=1}^K$ is a martingale.

By Azuma-Hoeffding's inequality: for any $\tau \le K$
\begin{equation}
\PP\left(\left|X_{t_\tau}\right| \ge \beta\right) \le 2\exp\left(-\frac{2\beta^2}{4\tau (H-h+1)^2}\right)
\end{equation}
Let $\frac{\delta}{2mK} = 2\exp\left(-\frac{2\beta^2}{4\tau (H-h+1)^2}\right)$, we get: for any $(s,a,h,j)$, for any $\tau \le K$, with probability at least $1-\frac{\delta}{2mK}$:
\begin{equation}
\left|\sum_{(s,a,r,s') \in D_h^{j,t_{\tau}}}
\left(
r+f(s') -   \left(\Bell_h f \right)(s,a)  
\right)
\right|
< \sqrt{\tau}(H-h+1)\sqrt{2\log\frac{4mK}{\delta}}.
\end{equation}
By union bound, 
for any $(s,a,h,j)$, with probability at least $1-\frac{\delta}{mK}$, for any $\tau \le K$:
\begin{equation}
\left|\sum_{(s,a,r,s') \in D_h^{j,t_{\tau}}}
\left(
r+f(s') -   \left(\Bell_h f \right)(s,a)  
\right)
\right|
< \sqrt{\tau}(H-h+1)\sqrt{2\log\frac{4mK}{\delta}}.
\end{equation}
This means for any $(s,a,h,j,k)$ and any $\tau \le k$
\begin{align}
&\PP\left(\left.\overline{\Ecal_{c-seq}(s,a,h,j,k)}\right| N_{h}^{j,k}(s,a)=\tau\right)\\
\le &
\PP\left(
\left.
\left|\left(\hat \Bell_h^{j,k} f \right)(s,a)  
-
\left(\Bell_h f \right)(s,a)  
\right|
\ge
\frac{H-h+1}{\sqrt{N_{h}^{j,k}(s,a)}}
\sqrt{2\log \frac{4Km}{\delta}} 
\right|
N_{h}^{j,k}(s,a) = \tau
\right) \\
\le & \frac{\delta}{mK}
\end{align}
Thus 
\begin{align}
\PP\left(\overline{\Ecal_{c-seq}(s,a,h,j,k)}\right) 
=&\sum_{\tau = 0}^{k}
\PP\left(\left.\overline{\Ecal_{c-seq}(s,a,h,j,k)}\right| N_{h}^{j,k}(s,a)=\tau\right)    \PP\left(N_{h}^{j,k}(s,a)=\tau\right) \\
\le & \frac{\delta}{2mK}
\end{align}
By union bound 
\begin{equation}
\PP\left({\Ecal_{conc-seq}(s,a,h,k)}\right) \ge 1-  \frac{\delta}{2} .
\end{equation}
\end{proof}

\begin{lemma}\label{lem:total seq dt conc}
Let:
\begin{equation}
\left(\hat \Bell_h^{j,k} f \right)(s,a)
:=
\frac{1}{N_{h}^{j,k}(s,a)}\sum_{(s,a,r,s') \in D_h^{j,k}}r+f(s'),
\end{equation}
\begin{equation}
\left(\hat \Bell_h^{\Gcal,k} f \right)(s,a)    
:=\frac{1}{\sum_{j\in\Gcal}\tilde N_h^{j,k}(s,a)}
\sum_{j\in\Gcal}\tilde N_h^{j,k}(s,a)\left(\hat \Bell_h^{j,k} f \right)(s,a),
\end{equation}
where we define $\frac{0}{0}=0$.
For any $f:\Scal \mapsto [H]$, with probability at least $1-\delta/2$, $\Ecal_{ct}(s,a,h,k)$ happens, 
where
\begin{align}
\Ecal_{ct}(s,a,h,k):=
\left\{
\left|\left(\hat \Bell_h^{\Gcal,k} f \right)(s,a)  
-
\left(\Bell_h f \right)(s,a)  
\right|
\le 
\frac{H-h+1}{\sqrt{\sum_{j\in\Gcal}\tilde N_{h}^{j,k}(s,a)}}
\sqrt{2\log \frac{4mK}{\delta}} 
\right\}
\end{align}
\end{lemma}
\begin{proof}[Proof \pref{lem:total seq dt conc}]
During the data collecting process, the agents are allowed to collect data simultaneously. 
For analysis purpose, we artificially order the data in the following sequence:
\begin{equation}\label{eq:E seq}
E^{1,1}, E^{2,1}, \ldots, E^{m,1}, E^{1,2}, \ldots, E^{m,2},
\ldots,
E^{1,K}, \ldots, E^{m,K}
\end{equation}
where $E^{j,k} := \left\{\left(s_{h}^{j,k},a_{h}^{j,k},r_{h}^{j,k},s_{h+1}^{j,k}\right)\right\}_{h=1}^H$.
Let 
\begin{equation}
\Fcal_t = \sigma\left(\bigcup_{j,k \mbox{ s.t. }m(k-1) + j \le t}E^{j,k}\right).
\end{equation}
Then $\left\{\Fcal_t\right\}_{t=0}^{mK}$ forms a valid filtration.
Let $\left\{\left\{\gamma_{j,k}\right\}_{j\in[m]}\right\}_{k\in[K]}$ be a fixed set of scalar, s.t. $0\le \gamma_{j,k}\le 1$, for all $j,k$.

For each fixed $(s,a,h) \in \Scal \x \Acal \x [H]$:
for all $t \in [mK]$, 
Let 
\begin{equation}
X_t =  \sum_{(s,a,r,s') \in \bigcup_{(j,k) \in \Gcal \x [K] \mbox{ s.t. }m(k-1) + j \le t}E^{j,k}}
\gamma_{j,k}
\left(
r+f(s') -   \left(\Bell_h f \right)(s,a)  
\right)
\end{equation}
Then  $\left\{\left(\Fcal_t, X_t\right)\right\}_{t=1}^{mK}$ is a martingale.
As we can see, if good agent $j$ did not visit $(s,a,h)$ in episode $k$, then $X_{m(k-1) + j} = X_{m(k-1) + j-1}$ a.s.
Thus we can use the stopping time idea to shorten the martingale sequence.

Define the following functions to map from sequence index to agent index and episode index:
\begin{align}
\mathcal{J}(t) := t - m\left(\lceil t/m \rceil - 1\right), \quad \mathcal{K}(t) := \lceil t/m \rceil
\end{align}
For any $n_1, \ldots, n_m$,
define the following sequence of $t_i$'s: $t_0:=0$,
\begin{align}
t_i := &\min\left(\left\{t' \in [mK]: t' \ge t_{i-1} \mbox{ and } (s_h^{\Jcal(t'), \Kcal(t')}, a_h^{\Jcal(t'), \Kcal(t')}) = (s,a) \right.\right.\\
&\left. \left.\mbox{ and for all }j\le \Jcal(t'), N_h^{j,\Kcal(t')} \le n_j;j> \Jcal(t'), N_h^{j,\Kcal(t')-1} \le n_j \right\}\cup \left\{K+1\right\}\right).
\end{align}
Intuitively, $t_i$ is the episode when $(s,a, h)$ is visited in sequence \pref{eq:E seq} for the $i$-th time.
And for all $j$, agent $j$ have not collected $n_j$ $(s,a,h)$ tuples.
If $(s,a,h)$ is visited for less than $i$ times or there exists agent $j$ visiting $(s,a,h)$ more than $n_j$ times, then $t_i = K+1$.
By definition, $t_i$ is a stopping time w.r.t. $\left\{\Fcal_t\right\}_{t=1}^{mK}$.

In particular, let $n_\cut$ be the $(2\alpha m  +1)$th-largest of all $n_j$'s and $\tilde n_j = \min(n_\cut, n_j)$.
We choose $\gamma_{j,k}:= \frac{\tilde n_j}{n_j}\le1$.

By optional sampling theorem,
$\left\{\left(\Fcal_{t_i}, X_{t_i}\right)\right\}_{i=1}^{mK}$ is a martingale.

By Azuma-Hoeffding's inequality: for any $\tau := \sum_{j\in[m]}n_j \le mK$
\begin{equation}
\PP\left(\left|X_{t_\tau}\right| \ge \beta\right) 
\le 2\exp\left(-\frac{2\beta^2}{4(H-h+1)^2\sum_{t=1}^{\tau}\gamma_{\Jcal(t), \Kcal(t)}^2}\right)
\end{equation}
Let $\frac{\delta}{2mK} = 2\exp\left(-\frac{2\beta^2}{4(H-h+1)^2\sum_{t=1}^{\tau}\gamma_{\Jcal(t), \Kcal(t)}^2}\right)$, we get: for any $(s,a,h)$, for any $\tau \le mK$, with probability at least $1-\frac{\delta}{2mK}$:
\begin{equation}
\left|X_{t_\tau}\right|
< \sqrt{\sum_{t=1}^{\tau}\gamma_{\Jcal(t), \Kcal(t)}^2}(H-h+1)\sqrt{2\log\frac{4mK}{\delta}}.
\end{equation}
By union bound, 
for any $(s,a,h)$, with probability at least $1-\frac{\delta}{2}$, for any $\tau \le mK$:
\begin{equation}
\left|X_{t_\tau}\right|
< \sqrt{\sum_{t=1}^{\tau}\gamma_{\Jcal(t), \Kcal(t)}^2}(H-h+1)\sqrt{2\log\frac{4mK}{\delta}}.
\end{equation}
This means for any $(s,a,h,k)\in\Scal \x \Acal \x [H] \x [
K]$ and any $\tau \le mk$
\begin{align}
&\PP\left(\left.\overline{\Ecal_{ct}(s,a,h,k)}\right| N_{h}^{j,k}(s,a)=n_{j}, \forall j\right)\\
\le &
\PP\left(
\left|\left(\hat \Bell_h^{\Gcal,k} f \right)(s,a)  
-
\left(\Bell_h f \right)(s,a)  
\right|
\ge
\frac{(H-h+1)\sqrt{\sum_{t=1}^{\tau}\frac{\tilde N_h^{\Jcal(t), \Kcal(t)}(s,a)}{N_h^{\Jcal(t), \Kcal(t)}(s,a)}}}{{\sum_{j\in\Gcal}\tilde N_{h}^{j,k}(s,a)}}\right.\\
&\left.\left.\cdot\sqrt{2\log \frac{4SAHmK^2}{\delta}} 
\right|
N_{h}^{j,k}(s,a)=n_{j}, \forall j
\right) \\
\le &
\PP\left(
\left|\left(\hat \Bell_h^{\Gcal,k} f \right)(s,a)  
-
\left(\Bell_h f \right)(s,a)  
\right|
\ge
\frac{(H-h+1)\sqrt{\sum_{t=1}^{\tau}\left(\frac{\tilde N_h^{\Jcal(t), \Kcal(t)}(s,a)}{N_h^{\Jcal(t), \Kcal(t)}(s,a)}\right)^2}}{{\sum_{j\in\Gcal}\tilde N_{h}^{j,k}(s,a)}}\right.\\
&\left.\left.\cdot\sqrt{2\log \frac{4SAHmK^2}{\delta}} 
\right|
N_{h}^{j,k}(s,a)=n_{j}, \forall j
\right) \\
\le & \frac{\delta}{2} \quad 
\left(\mbox{By $\gamma_{j,k} = \frac{\tilde N_h^{j,k}(s,a)}{N_h^{j,k}(s,a)}$}\right)
\end{align}
Thus 
\begin{align}
\PP\left(\overline{\Ecal_{ct}(s,a,h,k)}\right) 
=&\sum_{(n_1, \ldots, n_m) \in [K]^m}
\PP\left(\left.\overline{\Ecal_{ct}(s,a,h,k)}\right| N_{h}^{j,k}(s,a)=n_{j}, \forall j\right)  \\  &\PP\left(N_{h}^{j,k}(s,a)=n_{j}, \forall j\right) \\
\le & \frac{\delta}{2}
\end{align}
\end{proof}

\subsection{The regret decomposition for UCB style algorithm \label{sec:regret decom}}
We follow the regret decomposition strategy in \citep{jin2020provably} under event $\mathcal{E}$, i.e. the estimation error for the Bellman operator is bounded by the bonus term.

The estimated Bellman operator can be used to approximate the Q function:
\begin{lemma}\label{lem:q error}
Under event $\mathcal{E}$, for any $(s,a,h,k) \in \mathcal{S} \times \mathcal{A} \times H \times K$, and any policy $\pi'$
\begin{align}
\left|\left(\hat \Bell_h^k \hat V_{h+1}^k\right)(s,a) - Q_h^{\pi'}(s,a)  
- \EE_{s'\sim P_h(\cdot \mid s,a)}\left[\hat V_{h+1}^k (s') - V_{h+1}^{\pi'}(s')\right] \right| \le \Gamma_h^k(s,a)
\end{align}
\end{lemma}
\begin{proof}[Proof of \prettyref{lem:q error}]
\begin{align}
&\left|\left(\hat \Bell_h^k \hat V_{h+1}^k\right)(s,a) - Q_h^{\pi'}(s,a)  
- \EE_{s'\sim P_h(\cdot \mid s,a)}\left[\hat V_{h+1}^k (s') - V_{h+1}^{\pi'}(s')\right] \right| \\
\le & \left|\left(\hat \Bell_h^k \hat V_{h+1}^k\right)(s,a) - \left(\Bell_h \hat V_{h+1}^k\right)(s,a)\right| \\
&+\left|\left(\Bell_h \hat V_{h+1}^k\right)(s,a) -  \left( \Bell_h  V_{h+1}^{\pi'}\right)(s,a) - 
\EE_{s'\sim P_h(\cdot \mid s,a)}\left[\hat V_{h+1}^k (s') - V_{h+1}^{\pi'}(s')\right] \right| \\
& \left(\mbox{By triangular inequality and the fact that $\left( \Bell_h  V_{h+1}^{\pi'}\right)(s,a) = Q_h^{\pi'}(s,a)$.}\right)\\
\le & \Gamma_h^k(s,a) \\
& \left(\mbox{We can bound the first term by the definition of event $\mathcal{E}$,} \right.\\
&\left.\mbox{and the second term is zero by the definition of Bellman operator.}\right)
\end{align}
\end{proof}

Under event $\mathcal{E}$ we can upper bound the value function and Q function of the optimal policy by the estimated value function and Q function of policy $\hat \pi^k$:
\begin{lemma}[Optimism]\label{lem:opt}
Under event $\mathcal{E}$, 
$\forall s,a,h,k$:
\begin{align} 
\hat Q_h^k(s,a) \ge  Q_h^*(s,a), \quad
\hat V_h^k(s) \ge  V_h^*(s) \label{eq:hat Q ge Q}
\end{align}
\end{lemma}
\begin{proof}[Proof of \prettyref{lem:opt}]
We prove this by induction on $h$.
Before that, note that, for any $h,k,s$, if 
\begin{equation}
\hat Q_h^k(s,a) \ge  Q_h^*(s,a), \quad \forall a    
\end{equation}
then because $\hat \pi^k$ is chosen by maximizing $\hat Q_{h}^k(s,a)$,
we know
\begin{equation}
\hat V_h^k(s) = \max_a \hat Q_{h}^k(s,a) 
\ge \hat Q_{h}^k(s,\pi_h^*(a))  
\ge Q_{h}^*(s,\pi_h^*(a))
=V_h^*(s)
\end{equation}
This means for any $h,k,s$:
\begin{equation}\label{eq:Q 2 V}
\left\{\forall a,\hat Q_h^k(s,a) \ge  Q_h^*(s,a)\right\}   
\implies
\left\{\hat V_h^k(s) \ge V_h^*(s)\right\}
\end{equation}
We now begin our induction:
\begin{itemize}[leftmargin=5mm,itemsep=0pt]
\item For the base case, our goal is to show 
for any $s,a,k$,
in the last step $H$, 
\begin{align}
\hat Q_H^k(s,a) \ge  Q_H^*(s,a), \quad 
\hat V_H^k(s) \ge  V_H^*(s)
\end{align}
First note that $\hat V_{H+1} = V_{H+1}^*=0$.
By \prettyref{lem:q error} and choose $\pi' = \pi^*$, 
\begin{align}
\left|\left(\hat \Bell_H^k \hat V_{H+1}^k\right)(s,a) - Q_h^{*}(s,a)\right| \le \Gamma_H^k(s,a)
\end{align}
By definition of $\hat Q_H^k(s,a)$, and the fact that $Q_h^{*}(s,a)$ only contains the reward at step $H$, which is bounded by $1$:
\begin{align}
\hat Q_H^k(s,a) = \min\left(\left(\hat \Bell_H^k \hat V_{H+1}^k\right)(s,a) + \Gamma_H^k(s,a),1 \right)  
\ge Q_H^{*}(s,a)
\end{align}
By \prettyref{eq:Q 2 V}, $\hat V_H^k(s) \ge  V_H^*(s), \forall s$.
\item Suppose for any $s,a,k$, the statement holds for step $h+1$, i.e.
\begin{align} \label{eq:hat Q ge Q h+1}
\hat Q_{h+1}^k(s,a) \ge  Q_{h+1}^*(s,a), \quad
\hat V_{h+1}^k(s) \ge  V_{h+1}^*(s) 
\end{align}
our goal is to show
$\forall s,a,k$:
\begin{align} 
\hat Q_h^k(s,a) \ge  Q_h^*(s,a), \quad
\hat V_h^k(s) \ge  V_h^*(s)
\end{align}
\begin{align}
& \left(\hat \Bell_h^k \hat V_{h+1}^k\right)(s,a) + \Gamma_h^k(s,a)\\
\ge&\left(\hat \Bell_h^k \hat V_{h+1}^k\right)(s,a) +  \left|\left(\hat \Bell_h^k \hat V_{h+1}^k\right)(s,a) - Q_h^{*}(s,a)  
- \EE_{s'\sim P_h(\cdot \mid s,a)}\left[\hat V_{h+1}^k (s') - V_{h+1}^{*}(s')\right] \right| \\
& \left(\mbox{By \prettyref{lem:q error} and let $\pi' = \pi^*$}\right) \\
\ge & Q_h^{*}(s,a)  
+ \EE_{s'\sim P_h(\cdot \mid s,a)}\left[\hat V_{h+1}^k (s') - V_{h+1}^{*}(s')\right]\\
&\left(\mbox{By triangular inequality}\right)\\
\ge &  Q_h^{*}(s,a)  \\
&\left(\mbox{$\forall s,\hat V_{h+1}^k (s') \ge V_{h+1}^{*}(s')$ by \pref{eq:hat Q ge Q h+1}}\right)
\end{align}
By definition of Q function $Q_h^{*}(s,a) \le H-h+1$.
We concludes these two statements by:
\begin{align}
\hat Q_h^k(s,a) = \min\left(\left(\hat \Bell_h^k \hat V_{h+1}^k\right)(s,a) + \Gamma_h^k(s,a),H - h+1 \right)   \ge Q_h^{*}(s,a)
\end{align}
By \prettyref{eq:Q 2 V}, $\hat V_h^k(s) \ge  V_h^*(s), \forall s$.
\end{itemize}
\end{proof}

We are know ready to prove the regret decomposition lemma:
\begin{lemma}\label{lem:regret decomp}
Under good event $\mathcal{E}$:
\begin{align}
&\sum_{k=1}^K \sum_{j \in \mathcal{G}}\left(V_1^*(s_1) - V_1^{\hat\pi^k}(s_1)\right) \\
\le & 2\sum_{k=1}^K \sum_{j \in \mathcal{G}}\sum_{h=1}^H\Gamma_h^k(s_h^{j,k},a_h^{j,k}) \\
&+ 
\sum_{k=1}^K \sum_{j \in \mathcal{G}}\sum_{h=1}^H\left(\EE_{s'\sim P_h(\cdot \mid s_h^k,a_h^k)}\left[\hat V_{h+1}^k (s') - V_{h+1}^{\hat \pi^k}(s')\right] - \left(\hat V_{h+1}^k (s_{h+1}^{j,k}) - V_{h+1}^{\hat\pi^k}(s_{h+1}^{j,k})\right)\right)
\end{align}
\end{lemma}
\begin{proof}[Proof of \prettyref{lem:regret decomp}]
We start by showing the decomposition of regret after step $h$ in one episode of a single agent:
by \prettyref{lem:q error} and \prettyref{lem:opt}, under event $\mathcal{E}$, for any $s,k,h$
\begin{align}
&V_h^*(s) - V_h^{\hat\pi^k}(s)\le \hat V_h^{\hat \pi^k}(s) - V_h^{\hat\pi^k}(s)
\quad \left(\mbox{By \prettyref{lem:opt}}\right)
\\
= & \hat Q_h^{k}(s, \hat\pi_h^k(s)) - Q_h^{\hat \pi^k}(s, \hat\pi_h^k(s))\\
\le & \left(\hat \Bell_h^k \hat V_{h+1}^k\right)(s,a) + \Gamma_h^k(s,a)
- Q_h^{\hat \pi^k}(s, \hat\pi_h^k(s)) \quad \left(\mbox{By definition of $\hat Q_h^k$}\right)\\
\le&\left|\left(\hat \Bell_h^k \hat V_{h+1}^k\right)(s,\hat\pi_h^k(s)) - Q_h^{\hat\pi^k}(s,\hat\pi_h^k(s))  
- \EE_{s'\sim P_h(\cdot \mid s,\hat\pi_h^k(s))}\left[\hat V_{h+1}^k (s') - V_{h+1}^{\hat\pi^k}(s')\right] \right| \\
&+ \left|Q_h^{\hat\pi^k}(s,\hat\pi_h^k(s))  
+ \EE_{s'\sim P_h(\cdot \mid s,\hat\pi_h^k(s))}\left[\hat V_{h+1}^k (s') - V_{h+1}^{\hat\pi^k}(s')\right] \right| \\
& + \Gamma_h^k(s,a)- Q_h^{\hat \pi^k}(s, \hat\pi_h^k(s)) \\
&\left(\mbox{By using triangular inequality on the first term}\right)\\
\le& \Gamma_h^k(s,\hat\pi_h^k(s)) + 
Q_h^{\hat\pi^k}(s,\hat\pi_h^k(s))  
+ \EE_{s'\sim P_h(\cdot \mid s,\hat\pi_h^k(s))}\left[\hat V_{h+1}^k (s') - V_{h+1}^{\hat\pi^k}(s')\right] \\
&+ \Gamma_h^k(s,a)- Q_h^{\hat \pi^k}(s, \hat\pi_h^k(s))\\
& \left(\mbox{The first term is by using \prettyref{lem:q error} with $\pi' = \hat \pi^k$, }\right.\\
&\left.\mbox{the term inside the absolute in the second is non-negative by \prettyref{lem:opt}}\right)\\
= & 2\Gamma_h^k(s,\hat\pi_h^k(s))
+ \EE_{s'\sim P_h(\cdot \mid s,\hat\pi_h^k(s))}\left[\hat V_{h+1}^k (s') - V_{h+1}^{\hat\pi^k}(s')\right]
\end{align}

This indeed gives a recursive formula:
for any trajectory $\left\{(s_{h}^k, a_{h}^k, r_{h}^k, s_{h+1}^k)\right\}_{h\in[H]}$
\begin{align}
&\hat V_h^{\hat \pi^k}(s_h^k) - V_h^{\hat\pi^k}(s_h^k)     \\
\le& 
2\Gamma_h^k(s_h^k,\hat\pi_h^k(s_h^k))
+ \EE_{s'\sim P_h(\cdot \mid s_h^k,\hat\pi_h^k(s_h^k))}\left[\hat V_{h+1}^k (s') - V_{h+1}^{\hat\pi^k}(s')\right] \\
= & \hat V_{h+1}^k (s_h^k) - V_{h+1}^{\hat\pi^k}(s_h^k)
+ 2\Gamma_h^k(s_h^k,\hat\pi_h^k(s_h^k)) \\
&+ \left( \EE_{s'\sim P_h(\cdot \mid s_h^k,\hat\pi_h^k(s_h^k))}\left[\hat V_{h+1}^k (s') - V_{h+1}^{\hat\pi^k}(s')\right]
- \left(\hat V_{h+1}^k (s_h^k) - V_{h+1}^{\hat\pi^k}(s_h^k)\right)\right)
\end{align}

Then, we can show the regret decomposition in one episode of a single agent by recursion:

for any trajectory $\left\{(s_{h}^k, a_{h}^k, r_{h}^k, s_{h+1}^k)\right\}_{h\in[H]}$ collected by a clean agent under policy $\hat\pi^k$:
\begin{align}
& V_1^*(s_1^k) - V_1^{\hat\pi^k}(s_1^k) \le \hat V_1^k(s_1^k) - V_1^{\hat\pi^k}(s_1^k) \\
\le & \left(\hat V_{2}^k (s_2^k) - V_{2}^{\hat\pi^k}(s_2^k)\right) + 2\Gamma_1^k(s_1^k,a_1^k) \\
&+ 
\left( \EE_{s'\sim P_1(\cdot \mid s_1^k,a_1^k)}\left[\hat V_{2}^k (s') - V_{2}^{\hat\pi^k}(s')\right]
- \left(\hat V_{2}^k (s_2^k) - V_{2}^{\hat\pi^k}(s_2^k)\right)\right)\\
\le & \left(\hat V_{3}^k (s_3^k) - V_{3}^{\hat\pi^k}(s_3^k)\right) 
+ \sum_{h=1}^2 2\Gamma_h^k(s_h^k,a_h^k) \\
&+ 
\sum_{h=1}^2\left(\EE_{s'\sim P_h(\cdot \mid s_h^k,a_h^k)}\left[\hat V_{h+1}^k (s') - V_{h+1}^{\hat \pi^k}(s')\right] - \left(\hat V_{h+1}^k (s_{h+1}^k) - V_{h+1}^{\hat\pi^k}(s_{h+1}^k)\right)\right)\\
\le& \cdots\\
\le & \sum_{h=1}^H 2\Gamma_h^k(s_h^k,a_h^k) \\
&+ 
\sum_{h=1}^H\left(\EE_{s'\sim P_h(\cdot \mid s_h^k,a_h^k)}\left[\hat V_{h+1}^k (s') - V_{h+1}^{\hat \pi^k}(s')\right] - \left(\hat V_{h+1}^k (s_{h+1}^k) - V_{h+1}^{\hat\pi^k}(s_{h+1}^k)\right)\right)
\end{align}
Now we are ready to show the total regret decomposition.
For each episode, we can make the regret decomposition w.r.t. any trajectory collected by a clean agent following policy $\hat \pi^k$.
For convenience, we specialize the trajectories to be exactly the ones that are collected by the good agents and are used to calculate the bonus terms.
The purpose is, in the future, when we bound the regret, we need to bound the cumulative bonus used in the trajectory. 
By decomposing the regret w.r.t. the trajectory collected in the algorithm, 
it is naturally guaranteed that the $(s,a,h)$ tuples that collected a lot by the good agents have lower bonus.
This is because with more data collected, we can narrow down the confidence interval and design small but still valid bonus terms.

Because in our MDP definition, the MDP has a deterministic initial distribution, meaning the good agents always have the same starting state:
\begin{align}
&\sum_{k=1}^K \sum_{j \in \mathcal{G}}\left(V_1^*(s_1) - V_1^{\hat\pi^k}(s_1)\right)
= \sum_{k=1}^K \sum_{j \in \mathcal{G}}\left(V_1^*(s_1^{j,k}) - V_1^{\hat\pi^k}(s_1^{j,k})\right) \\
\le & 2\sum_{k=1}^K \sum_{j \in \mathcal{G}}\sum_{h=1}^H\Gamma_h^k(s_h^{j,k},a_h^{j,k}) \\
&+ 
\sum_{k=1}^K \sum_{j \in \mathcal{G}}\sum_{h=1}^H\left(\EE_{s'\sim P_h(\cdot \mid s_h^{j,k},a_h^{j,k})}\left[\hat V_{h+1}^k (s') - V_{h+1}^{\hat \pi^k}(s')\right] - \left(\hat V_{h+1}^k (s_{h+1}^{j,k}) - V_{h+1}^{\hat\pi^k}(s_{h+1}^{j,k})\right)\right)
\end{align}
\end{proof}

\subsection{Evenness of clean agents\label{sec:even agent eve}}
We need at least $(2\alpha m + 1)$-agents to cover $(s,a,h)$ in order to learn the Bellman operator properly. 
In this section, we show that the agents have ``even'' coverage on the visited $(s,a,h)$ tuples in each (except a relatively small number) of the episodes. 
In the following we use 
$\tilde m := (1-\alpha)m = |\mathcal{G}|$ 
to denote the number of good agents.

Formally, we have:
\begin{lemma}[Even coverage of good agent]\label{lem:evenness of good agent}
For any $(s,a,h,k) \in \mathcal{S}\times\mathcal{A}\times [H] \times [K]$,
we define the follow event:
\begin{equation}
\mathcal{E}_{\mbox{even}}(s,a, h, k) := \left\{\mbox{if $\sum_{j\in \mathcal{G}} N_h^{j,k}(s,a)
\ge 400m\log \frac{2mKSAH}{\delta} 
$, then $\max_{i,j \in \mathcal{G}}\frac{N_h^{j,k}(s,a)}{N_h^{i,k}(s,a)} \le 2$}\right\}    
\end{equation}
then, we have:
for all $0<\delta < \frac{1}{4}$
\begin{equation}
\PP\left(\bigcap_{(s,a,h,k) \in \mathcal{S}\times\mathcal{A}\times [H] \times [K]} \mathcal{E}_{\mbox{even}}(s,a, h, k)\right) \ge 1 - 2\delta    
\end{equation}
\end{lemma}
\begin{remark}[Intuition of the good event]
The event $\mathcal{E}_{\mbox{even}}(s,a, h, k)$ characterizes that: 
if in any episode $k$, a $(s,a,h)$ tuple gets enough coverage from the clean agents, then the coverage of each agent are very close.
\end{remark}
See proof of \prettyref{lem:evenness of good agent} in \prettyref{sec:pf even}.

\subsubsection{Proof of \prettyref{lem:evenness of good agent}\label{sec:pf even}}
Proof of \prettyref{lem:evenness of good agent} depends on the concentration of $N_h^{j,k}(s,a)$:
\begin{lemma}[Concentration of counts around empirical mean]\label{lem:concen of cts}
For all $0<\delta < \frac{1}{4}$
\begin{align}
&\PP\left(\bigcup_{s,a,h,k,j}
 \left\{ \left|N_h^{j,k}(s,a) - \frac{1}{|\Gcal|}\sum_{j \in \mathcal{G}}N_h^{j,k}(s,a)\right| 
\right.\right. \\
&> \left.\left.
18\log \frac{2SAHmK}{\delta} +
4\sqrt{\log\frac{2SAHmK}{\delta}}
\sqrt{\frac{1}{|\Gcal|}\sum_{j\in\Gcal} N_h^{j,k}(s,a)}
\right\}\right) < 2\delta   
\end{align}
\end{lemma}
\begin{proof}[Proof of \pref{lem:concen of cts}]
See \pref{sec:concen of cts}.
\end{proof}

\begin{proof}[Proof of \prettyref{lem:evenness of good agent}]
Let 
\begin{equation}
N_0:= 400m\log \frac{2mKSAH}{\delta} 
\end{equation}
For any $(s,a,h,k) \in \mathcal{S} \times \mathcal{A} \times [H] \times [K]$, define events:
\begin{align}
\mathcal{E}_1(s,a,h,k) := & \left\{\sum_{j\in\mathcal{G}} N_h^{j,k}(s,a)\ge N_0\right\}    \\
\mathcal{E}_2(s,a,h,k) := & \left\{\max_{i,j \in \mathcal{G}}\frac{N_h^{j,k}(s,a)}{N_h^{i,k}(s,a)} \le 2\right\}    
\end{align}
Recall:
\begin{equation}
\mathcal{E}_{\mbox{even}}(s,a, h, k) := \left\{\mbox{if $\sum_{j\in \mathcal{G}} N_h^{j,k}(s,a)
\ge 400m\log \frac{2mKSAH}{\delta} 
$, then $\max_{i,j \in \mathcal{G}}\frac{N_h^{j,k}(s,a)}{N_h^{i,k}(s,a)} \le 2$}\right\}    
\end{equation}
Then we can rewrite even $\mathcal{E}_{\mbox{even}}(s,a, h, k)$ as:
\begin{equation}
\mathcal{E}_{\mbox{even}}(s,a, h, k) = \overline{\mathcal{E}_1(s,a,h,k)} \cup \mathcal{E}_2(s,a,h,k)
\end{equation}
We first show that of there are two $N_h^{j,k}$'s, whose ratio exceed $2$, 
then there must be some $N_h^{j,k}$ that deviates a lot from the empirical mean of $N_h^{j,k}$'s:
\begin{align}
&\overline{\mathcal{E}_2(s,a,h,k)}
= \left\{\max_{i,j \in \mathcal{G}}\frac{N_h^{j,k}(s,a)}{N_h^{i,k}(s,a)} > 2\right\}  \\
\subseteq &\bigcup_{i \in \mathcal{G}} \left\{N_h^{i,k}(s,a) > \frac{498}{400}\frac{1}{|\Gcal|}\sum_{j \in \mathcal{G}}N_h^{j,k}(s,a)\right\}
\cup \bigcup_{i \in \mathcal{G}} \left\{N_h^{i,k}(s,a) < \frac{302}{400}\frac{1}{|\Gcal|}\sum_{j \in \mathcal{G}}N_h^{j,k}(s,a)\right\} \\
= &\bigcup_{i \in \mathcal{G}} \left\{N_h^{i,k}(s,a) - \frac{1}{|\Gcal|}\sum_{j \in \mathcal{G}}N_h^{j,k}(s,a)> \frac{98}{400}\frac{1}{|\Gcal|}\sum_{j \in \mathcal{G}}N_h^{j,k}(s,a)\right\} \\
&\cup \bigcup_{i \in \mathcal{G}} \left\{N_h^{i,k}(s,a) - \frac{1}{|\Gcal|}\sum_{j \in \mathcal{G}}N_h^{j,k}(s,a)< -\frac{98}{400}\frac{1}{|\Gcal|}\sum_{j \in \mathcal{G}}N_h^{j,k}(s,a)\right\} \\
=& \bigcup_{i \in \mathcal{G}} \left\{\left|N_h^{i,k}(s,a) - \frac{1}{|\Gcal|}\sum_{j \in \mathcal{G}}N_h^{j,k}(s,a)\right|> \frac{98}{400}\frac{1}{|\Gcal|}\sum_{j \in \mathcal{G}}N_h^{j,k}(s,a)\right\} \label{eq:deviate from mean}
\end{align}
To show that $\mathcal{E}_{\mbox{even}}(s,a, h, k)$ happens w.h.p.:
\begin{align}
&\PP\left(\bigcup_{s,a,h,k}\overline{\mathcal{E}_{\mbox{even}}(s,a, h, k)}\right)  
= \PP\left(\bigcup_{s,a,h,k}\overline{\overline{\mathcal{E}_1(s,a,h,k)} \cup \mathcal{E}_2(s,a,h,k)}\right)\\
= & \PP\left(\bigcup_{s,a,h,k}{\mathcal{E}_1(s,a,h,k)} \cap \overline{\mathcal{E}_2(s,a,h,k)}\right)\\
\le & \PP\left(\exists s,a,h,k,\sum_{j\in\mathcal{G}} N_h^{j,k}(s,a)\ge N_0, \right.\\
&\left.\exists i \in \mathcal{G}, 
\left|N_h^{i,k}(s,a) - \frac{1}{|\Gcal|}\sum_{j\in\mathcal{G}} N_h^{j,k}(s,a)\right| > \frac{98}{400}\frac{1}{|\Gcal|}\sum_{j\in\mathcal{G}} N_h^{j,k}(s,a)\right) \\
& \mbox{(By \prettyref{eq:deviate from mean})} \\
\le & \PP\left(\exists s,a,h,k,i\left|N_h^{i,k}(s,a) - \frac{1}{|\Gcal|}\sum_{j\in\mathcal{G}} N_h^{j,k}(s,a)\right| \right.\\
&\left.> \frac{18}{400}\frac{1}{|\Gcal|}N_0 + 4\sqrt{\frac{1}{400}\frac{1}{|\Gcal|}N_0}\sqrt{\frac{1}{|\Gcal|}\sum_{j\in\mathcal{G}} N_h^{j,k}(s,a)}\right) \\
= & \PP\left(\exists s,a,h,k,i\left|N_h^{i,k}(s,a) - \frac{1}{|\Gcal|}\sum_{j\in\mathcal{G}} N_h^{j,k}(s,a)\right| \right.\\
&\left.> 18\log \frac{2 mKSAH}{\delta} 
+ 4\sqrt{\log \frac{2 mKSAH}{\delta}}
\sqrt{\frac{1}{|\Gcal|}\sum_{j\in\mathcal{G}} N_h^{j,k}(s,a)}
\right) \\
< &2\delta \quad \mbox{(By \prettyref{lem:concen of cts})}
\end{align}
\end{proof}
\subsubsection{Proof of \prettyref{lem:concen of cts}\label{sec:concen of cts}}
The high level ideas are:
\begin{enumerate}
\item For each $s,a,h$, 
\begin{itemize}
\item for each $j\in\mathcal{G}$, define centered $N_h^{j,k}(s,a)$ as a martingale;
\item define centered
$\sum_{j\in\mathcal{G}} N_h^{j,k}(s,a)$ as a martingale;
\end{itemize} 
\item apply a modified Bernstein type of martingale concentration bound for both centered $N_h^{j,k}(s,a)$'s and
centered $\sum_{j\in\mathcal{G}} N_h^{j,k}(s,a)$ (see \pref{lem:Nj conc} and \pref{lem:Nj ave conc});
\item \label{step:bd dist} because $N_h^{j,k}(s,a)$ and $\frac{1}{\tilde m}\sum_{j\in\mathcal{G}} N_h^{j,k}(s,a)$ have the same mean, we can use triangular inequality to show these two terms are close, and the distance is bounded by the variance term in Bernstein inequality.
\item \label{step:bd var}Bernstein on $\frac{1}{\tilde m}\sum_{j\in\mathcal{G}} N_h^{j,k}(s,a)$ also allow us to bound its variance in terms of itself.
\item We can get our result by combining \prettyref{step:bd dist} and \prettyref{step:bd var}.
\end{enumerate}
\begin{lemma}[Concentration of each $N_h^{j,k}(s,a)$]\label{lem:Nj conc}
For all $0<\delta \le1/4$, with probability at least $1-\delta$, for all 
$(s,a,h,j,k)\in\Scal\x\Acal\x[H]\x\Gcal\x[K]$:
\begin{equation}\label{eq:nj conc}
\left|N_h^{j,k}(s,a) - \sum_{t=1}^k d_h^{\hat \pi^{t}}(s,a)\right| < 3\log \frac{2SAHmK}{\delta}    + \sqrt{2\sum_{t=1}^k d_h^{\hat \pi^{t}}(s,a)\log\frac{2SAHmK}{\delta}}
\end{equation}
\end{lemma}
\begin{proof}[Proof of \pref{lem:Nj conc}]
See \pref{sec:Nj conc}
\end{proof}
\begin{lemma}[Concentration of each $\frac{1}{\tilde m}\sum_{j\in\mathcal{G}} N_h^{j,k}(s,a)$]\label{lem:Nj ave conc}
For all $0<\delta \le1/4$,
with probability at least $1-\delta$, for all $(s,a,h,k) \in \Scal\x\Acal\x[H]\x[K]$:
\begin{equation}\label{eq:nj ave conc}
\left|
\sum_{j\in\Gcal} N_h^{j,k}(s,a)- 
|\Gcal|\sum_{t=1}^k d_h^{\hat \pi^t}(s,a)
\right|
< 
3\log \frac{2SAHmK}{\delta} 
+ \sqrt{2|\Gcal|\sum_{t=1}^k d_h^{\hat \pi^{t}}(s,a)\log\frac{2SAHmK}{\delta}}    
\end{equation}
\end{lemma}
\begin{proof}[Proof of \pref{lem:Nj ave conc}]
See \pref{sec:Nj ave conc}
\end{proof}
\begin{proof}[Proof of \prettyref{lem:concen of cts}]
Let $\Ecal_{N}$ the intersection of the events in \pref{lem:Nj conc} and \pref{lem:Nj ave conc}. 
Then by \pref{lem:Nj conc} and \pref{lem:Nj ave conc}, $\Ecal_{N}$ happens with probability at least $1-2\delta$.
By \pref{eq:nj ave conc}, 
\begin{equation}\label{eq:sum d bd}
\sqrt{\sum_{t=1}^k d_h^{\hat \pi^{t}}(s,a)}
\le 
4\sqrt{\log \frac{2SAHmK}{\delta} }
+
\sqrt{\frac{1}{|\Gcal|}\sum_{j\in\Gcal} N_h^{j,k}(s,a)}
\end{equation}
By \pref{eq:nj conc} and \pref{eq:nj ave conc}, for all $s,a,h,j,k$
\begin{align}
&  \left|
\frac{1}{|\Gcal|}\sum_{j'\in\Gcal} N_h^{j',k}(s,a)- 
N_h^{j,k}(s,a)
\right|   \\
\le & 
\left|N_h^{j,k}(s,a) - \sum_{t=1}^k d_h^{\hat \pi^{t}}(s,a)\right|
+ 
\left|
\frac{1}{|\Gcal|}\sum_{j'\in\Gcal} N_h^{j',k}(s,a)- 
\sum_{t=1}^k d_h^{\hat \pi^t}(s,a)
\right|\\
\le & 
6\log \frac{2SAHmK}{\delta}    + 2\sqrt{2\sum_{t=1}^k d_h^{\hat \pi^{t}}(s,a)\log\frac{2SAHmK}{\delta}}\\
\le & 6\log \frac{2SAHmK}{\delta} +
2\sqrt{2\log\frac{2SAHmK}{\delta}}
\left(4\sqrt{\log \frac{2SAHmK}{\delta} }
+
\sqrt{\frac{1}{|\Gcal|}\sum_{j\in\Gcal} N_h^{j,k}(s,a)}\right)\\
\le & 18\log \frac{2SAHmK}{\delta} +
4\sqrt{\log\frac{2SAHmK}{\delta}}
\sqrt{\frac{1}{|\Gcal|}\sum_{j\in\Gcal} N_h^{j,k}(s,a)}
\end{align}
\end{proof}

\subsubsection{Proof of \pref{lem:Nj conc}\label{sec:Nj conc}}
\begin{proof}[Proof of \pref{lem:Nj conc}]
For each fixed $(s,a,h,j) \in \Scal \x \Acal \x [H] \x \Gcal$:
for all $t \in [K]$, define
\begin{equation}
\Fcal_k := \sigma\left(\bigcup_{t\le k}\bigcup_{j \in [m]}\left\{\left(s_{h}^{j,t},a_{h}^{j,t},r_{h}^{j,t},s_{h+1}^{j,t}\right)\right\}_{h=1}^H\right).    
\end{equation}
Let 
\begin{align}
S_h^{j,k}(s,a) = & N_h^{j,k}(s,a) - \sum_{t=1}^k d_h^{\hat \pi^{t}}(s,a)\\
T_h^{j,k}(s,a) = &\sum_{t=1}^k d_h^{\hat \pi^{t}}(s,a)\left(1-d_h^{\hat \pi^{t}}(s,a)\right)
\end{align}
Then  
$\left\{\left(\Fcal_k, S_h^{j,k}(s,a)\right)\right\}_{t=k}^K$ 
is a martingale.
Since $d_h^{\hat \pi^{k}}(s,a)$ depends on $\hat \pi^k$, which is calculated use data in the first $k-1$ episodes, then $d_h^{\hat \pi^{k}}(s,a) \in \Fcal_{k-1}$.
By \pref{corr:bern var},
\begin{align}
&\PP\left(\bigcup_{k=1}^K\left\{|S_h^{j,k}(s,a)| \ge 3\log \frac{2SAHmK}{\delta}    + \sqrt{2\sum_{t=1}^k d_h^{\hat \pi^{t}}(s,a)\log\frac{2SAHmK}{\delta}}\right\}
\right)\\
\le & \PP\left(\bigcup_{k=1}^K\left\{|S_h^{j,k}(s,a)| \ge 3\log \frac{2SAHmK}{\delta}    + \sqrt{2T_h^{j,k}(s,a)\log\frac{2SAHmK}{\delta}}\right\}
\right)
    \le \frac{\delta}{SAHm}    
\end{align}
By union bound, with probability at least $1-\delta$, for all 
$(s,a,h,j,k)\in\Scal\x\Acal\x[H]\x\Gcal\x[K]$:
\begin{equation}
|S_h^{j,k}(s,a)| < 3\log \frac{2SAHmK}{\delta}    + \sqrt{2\sum_{t=1}^k d_h^{\hat \pi^{t}}(s,a)\log\frac{2SAHmK}{\delta}}
\end{equation}
\end{proof}

\subsubsection{Proof of \pref{lem:Nj ave conc}\label{sec:Nj ave conc}}
\begin{proof}[Proof of \pref{lem:Nj ave conc}]
During the data collecting process, the agents are allowed to collect data simultaneously. 
For analysis purpose, we artificially order the data in the following sequence:
\begin{equation}\label{eq:E seq Njs}
E^{1,1}, E^{2,1}, \ldots, E^{m,1}, E^{1,2}, \ldots, E^{m,2},
\ldots,
E^{1,K}, \ldots, E^{m,K}
\end{equation}
where $E^{j,k} := \left\{\left(s_{h}^{j,k},a_{h}^{j,k},r_{h}^{j,k},s_{h+1}^{j,k}\right)\right\}_{h=1}^H$.
Let 
\begin{equation}
\Fcal_t = \sigma\left(\bigcup_{j,k \mbox{ s.t. }m(k-1) + j \le t}E^{j,k}\right).
\end{equation}
Then $\left\{\Fcal_t\right\}_{t=0}^{mK}$ forms a valid filtration.
Define the following functions to map from sequence index to agent index and episode index:
\begin{align}
\mathcal{J}(t) := t - m\left(\lceil t/m \rceil - 1\right), \quad \mathcal{K}(t) := \lceil t/m \rceil
\end{align}

For each fixed $(s,a,h) \in \Scal \x \Acal \x [H]$, 
for all $t \in [mK]$, we define $S_h^{\Gcal,t}(s,a)$ as the (centered) total counts of $(s,a,h)$ collected by all good agents up to time $t$.
The $t$-th term in \pref{eq:E seq Njs} could be in the center of an episode, meaning some agents have not collected their trajectories yet. 
So we need to treat the agents differently:
Let 
\begin{align}
S_h^{\Gcal,t}(s,a) = &
\sum_{j\in \Gcal, j \le \Jcal(t)}
\left(
N_h^{j,\Kcal(t)}(s,a) - \sum_{t=1}^{\Kcal(t)} d_h^{\hat \pi^{t}}(s,a)
\right)\\
&+
\sum_{j\in \Gcal, j > \Jcal(t)}
\left(
N_h^{j,\Kcal(t)-1}(s,a) - \sum_{t=1}^{\Kcal(t)-1} d_h^{\hat \pi^{t}}(s,a)
\right)
\end{align}
Then  $\left\{\left(\Fcal_t, S_h^{\Gcal,t}(s,a)\right)\right\}_{t=1}^{mK}$ is a martingale.
Similar to \pref{lem:Nj conc}, define
\begin{align}
T_h^{\Gcal,t}(s,a) = &
\sum_{j\in \Gcal, j \le \Jcal(t)}
\sum_{t=1}^{\Kcal(t)} d_h^{\hat \pi^{t}}(s,a)\left(1-d_h^{\hat \pi^{t}}(s,a)\right)\\
&+
\sum_{j\in \Gcal, j > \Jcal(t)}
\sum_{t=1}^{\Kcal(t)-1} d_h^{\hat \pi^{t}}(s,a)\left(1-d_h^{\hat \pi^{t}}(s,a)\right)
\end{align}
Then by \pref{corr:bern var},
\begin{align}
&\PP\left(\bigcup_{k\in[K]}
\left\{
\left|
\sum_{j\in\Gcal} N_h^{j,k}(s,a)- 
|\Gcal|\sum_{t=1}^k d_h^{\hat \pi^t}(s,a)
\right|
\ge 
3\log \frac{2SAHmK}{\delta}    \right.\right.\\
& \left.\left.+ \sqrt{2|\Gcal|\sum_{t=1}^k d_h^{\hat \pi^{t}}(s,a)\log\frac{2SAHmK}{\delta}}
\right\}
\right)  \\
\le &\PP\left(\bigcup_{k\in[K]}
\left\{
\left|
\sum_{j\in\Gcal} N_h^{j,k}(s,a)- 
|\Gcal|\sum_{t=1}^k d_h^{\hat \pi^t}(s,a)
\right|
\ge 
3\log \frac{2SAHmK}{\delta}    \right.\right.\\
& \left.\left.+ \sqrt{2|\Gcal|\sum_{t=1}^k d_h^{\hat \pi^{t}}(s,a)\left(1-d_h^{\hat \pi^{t}}(s,a)\right)\log\frac{2SAHmK}{\delta}}
\right\}
\right)  \\
\le & \PP\left(\bigcup_{k=1}^{mK}\left\{|S_h^{\Gcal,k}(s,a)| \ge 3\log \frac{2SAHmK}{\delta}    + \sqrt{2T_h^{\Gcal,k}(s,a)\log\frac{2SAHmK}{\delta}}\right\}
\right) \\
\le & \frac{\delta}{SAH}  
\end{align}
By union bound, with probability at least $1-\delta$, for all $(s,a,h,k) \in \Scal\x\Acal\x[H]\x[K]$:
\begin{equation}
\left|
\sum_{j\in\Gcal} N_h^{j,k}(s,a)- 
|\Gcal|\sum_{t=1}^k d_h^{\hat \pi^t}(s,a)
\right|
< 
3\log \frac{2SAHmK}{\delta} 
+ \sqrt{2|\Gcal|\sum_{t=1}^k d_h^{\hat \pi^{t}}(s,a)\log\frac{2SAHmK}{\delta}}    
\end{equation}
\end{proof}

\section{Proof of \prettyref{thm:off-line VI}\label{sec:proof offline vi}}
By the following lemma, we can upper bound the suboptimality by the cumulative bonuses:
\begin{lemma}\label{lem:subopt pess VI}[Suboptimality for Pessimistic Value Iteration, Lemma 3.2 in \citep{zhang2021corruption} and Theorem 4.2 in \citep{jin2021pessimism}]\label{lem:pess_opt}
	Under the event $\mathcal{E}$ that the $\Gamma_h(s,a)$ satisfies the required property of bounding the Bellman error, i.e.
	$
		|\hat Q_h(s,a)-(\Bell\hat V_{h+1})(s,a)| \leq \Gamma_h(s,a), \forall h\in[H], (s,a) \in \mathcal{S} \times \mathcal{A}
	$
then against any comparator policy $\tilde\pi$, it achieves
	\begin{align}\label{eq:pess_opt}
		\subopt(\hat\pi,\tilde\pi)\leq 2\sum_{h=1}^H\EE_{d^{\tilde\pi}}[\Gamma_h(s_h,a_h)]
	\end{align}
\end{lemma}
Recall that for all $(s,a,h)\in\Scal\x\Acal\x[H]$,
\begin{equation}
N_h^j(s,a) := \sum_{k \in [K_j]}\one\left\{(s_h^{j,k}, a_h^{j,k}) = (s, a)\right\}, \quad \forall j \in [m].
\end{equation}
and  
$N_h^{\operatorname{cut}}(s,a)$ is the $(2\alpha m +1)$-largest among $\left\{N_h^j(s,a)\right\}_{j\in[m]}$.
$N_h^{\Gcal, \cut_1}(s,a)$ is
the $(\alpha m+1)$-th largest of $\left\{N_h^j(s,a)\right\}_{j\in\Gcal}$
and $N_h^{\Gcal, \cut_2}(s,a)$ is the $(2\alpha m+1)$-th largest of $\left\{N_h^j(s,a)\right\}_{j\in\Gcal}$.
The bonuses are given by:
\begin{itemize}[leftmargin=5mm,itemsep=0pt]
\item If $N_h^{\operatorname{cut}}(s,a)=0$
\begin{equation}
\Gamma_h(s,a) = H - h + 1;    
\end{equation}
\item If $N_h^{\operatorname{cut}}(s,a)>0$
\begin{align}
\Gamma_h(s,a) := & 
\frac{2{(H-h+1)}}{\sqrt{\sum_{j \in [m]}\tilde N_h^{j}(s,a)}}
\sqrt{2\log\frac{2SAH}{\delta}} \\
&+ \frac{8\alpha m\sqrt{N_h^{\operatorname{cut}}(s,a)}}{\sum_{j \in [m]}\tilde N_h^{j}(s,a)}
{(H-h+1)}\sqrt{2\log\frac{2mSAH}{\delta}} 
\end{align}
Where 
\begin{equation}
\tilde N_h^{j}(s,a) = \max\left(N_h^{\cut}(s,a), N_h^{j}(s,a)\right).    
\end{equation}
\end{itemize}
\begin{proof}[Proof of \prettyref{thm:off-line VI}]
We first show that with probability at least $1-\delta$,
\begin{equation}\label{eq:pess bd}
|(\hat\Bell_h\hat V_{h+1})(s,a)-(\Bell_h\hat V_{h+1})(s,a)| \leq \Gamma_h(s,a), \quad
\forall (s,a) \in \mathcal{S}\times \mathcal{A}, \forall h\in[H]
\end{equation}
where $\Gamma_h(s,a)$ is defined in \prettyref{eq:pess_opt}.
\begin{itemize}[leftmargin=5mm,itemsep=0pt]
\item if $N_h^{\cut}(s,a) = 0$, by definition, $(\hat \Bell_h\hat V_{h+1})(s,a) = 0$. By definition of $\hat V_h$ and $\Bell_h$, $(\Bell_h\hat V_{h+1})(s,a) \in [0, H-h+1]$, thus \prettyref{eq:pess bd} holds;
\item if $N_h^{\cut}(s,a) > 0$, for any fixed $h \in [H]$, $(s,a) \in \mathcal{S} \times \mathcal{A}$, $f: \mathcal{S} \to [0,H]$.
Because $(\hat\Bell_h f)(s,a)$ is bounded and thus sub-Gaussian, we can use \prettyref{thm:weighted int} to upper bound $|(\hat\Bell_h f)(s,a)-(\Bell_h f)(s,a)|$:
\begin{align}
\PP\left( \left|(\hat\Bell_h f)(s,a)-(\Bell_h f)(s,a)\right|
\ge \Gamma_h(s,a) \right) \le \frac{\delta}{HSA}
\end{align}

Thus 
\begin{align}
&\PP\left(|(\hat\Bell_h\hat V_{h+1})(s,a)-(\Bell_h\hat V_{h+1})(s,a)| 
\ge \Gamma_h(s,a)\right) \\
= & 
\int_{[0,H]^\Scal} \PP\left( \left.|(\hat\Bell_h\hat V_{h+1})(s,a)-(\Bell_h\hat V_{h+1})(s,a)| 
\ge \Gamma_h(s,a) \right| \hat V_{h+1}(\cdot)\right) d \PP(\hat V_{h+1}(\cdot)) \\
\le & \frac{\delta}{HSA}
\end{align}
By union bound, with probability at least $1-\delta$,
\begin{equation}
|(\hat\Bell_h\hat V_{h+1})(s,a)-(\Bell_h\hat V_{h+1})(s,a)| \leq \Gamma_h(s,a), \quad
\forall (s,a) \in \mathcal{S}\times \mathcal{A}, \forall h\in[H]
\end{equation}
\end{itemize}


Then, by \prettyref{lem:subopt pess VI}, with probability at least $1-\delta$,
\begin{align}\label{eq:pess_opt in proof}
&\subopt(\hat\pi,\tilde\pi)\leq 2\sum_{h=1}^H\EE_{d^{\tilde\pi}}[\Gamma_h(s_h,a_h)] \\
=& 2\sum_{h=1}^H\EE_{d^{\tilde\pi}}\left[\Gamma_h(s_h,a_h)\one{\left\{N_h^{\Gcal,\cut_2}(s_h,a_h) = 0\right\}}\right] \\
&+ 2\sum_{h=1}^H\EE_{d^{\tilde\pi}}\left[\Gamma_h(s_h,a_h)\one\left\{N_h^{\Gcal,\cut_2}(s_h,a_h) > 0\right\}\right] \\
=:&\mathscr{A}_1 + \mathscr{A}_2.
\end{align}
By definition of $p^{\Gcal,0}$ in \pref{def:cover},
\begin{equation}
\mathscr{A}_1 \le 2H p^{\Gcal,0}   
\end{equation}
\begin{align}
\mathscr{A}_2 = & 2\sum_{h=1}^H\EE_{d^{\tilde\pi}}\left[\Gamma_h(s_h,a_h)\one\left\{N_h^{\Gcal,\cut_2}(s_h,a_h) > 0\right\}\right]  \\
\le & 2\sum_{h=1}^H\EE_{d^{\tilde\pi}}\left[\left(\frac{2{(H-h+1)}}{\sqrt{\sum_{j \in \Gcal}\tilde N_h^{j}(s,a)}}
\sqrt{2\log\frac{2SAH}{\delta}}\right.\right. \\
&\left.\left.+ \frac{8\alpha m\sqrt{N_h^{\operatorname{cut}}(s,a)}}{\sum_{j \in \Gcal}\tilde N_h^{j}(s,a)}
{(H-h+1)}\sqrt{2\log\frac{2mSAH}{\delta}} \right)\one\left\{N_h^{\Gcal,\cut_2}(s_h,a_h) > 0\right\}\right].
\end{align}
By the definition of $\kappa_{\text{even}}$ in \pref{def:balance}: for $a = \tilde \pi(s)$,
\begin{align}
\frac{1}{\sqrt{\sum_{j \in \Gcal}\tilde N_h^{j}(s,a)}}
= & \frac{\sqrt{\sum_{j \in \Gcal}N_h^{j}(s,a)}}{\sqrt{\sum_{j \in \Gcal}\tilde N_h^{j}(s,a)}}
\frac{1}{\sqrt{\sum_{j \in \Gcal}N_h^{j}(s,a)}} \\
\le & \frac{\sqrt{\sum_{j \in \Gcal}N_h^{j}(s,a)}}{\sqrt{\sum_{j \in \Gcal}\tilde N_h^{j,\cut_2}(s,a)}}
\frac{1}{\sqrt{\sum_{j \in \Gcal}N_h^{j}(s,a)}}\\
\le &\frac{\sqrt{\kappa_{\text{even}}}}{\sqrt{\sum_{j \in \Gcal}N_h^{j}(s,a)}}
\end{align}
and 
\begin{align}
\frac{ m\sqrt{N_h^{\operatorname{cut}}(s,a)}}{\sum_{j \in \Gcal}\tilde N_h^{j}(s,a)}  
&\le\frac{1}{\sqrt{1-\alpha}}
\sqrt{
\frac{\sum_{j \in \Gcal}N_h^{j}(s,a)}{\sum_{j \in \Gcal}\tilde N_h^{j}(s,a)} 
\frac{ m(1-\alpha){N_h^{\operatorname{cut}}(s,a)}}{\sum_{j \in \Gcal}\tilde N_h^{j}(s,a)} 
}
\frac{\sqrt{m}}{\sqrt{\sum_{j \in \Gcal}N_h^{j}(s,a)}}\\
&\le
\sqrt{
\frac{\sum_{j \in \Gcal}N_h^{j}(s,a)}{\sum_{j \in \Gcal}\tilde N_h^{j,\cut_2}(s,a)} 
\frac{ m(1-\alpha){N_h^{\Gcal,\cut_1}(s,a)}}{\sum_{j \in \Gcal}\tilde N_h^{j,\cut_2}(s,a)} 
}
\frac{\sqrt{2m}}{\sqrt{\sum_{j \in \Gcal}N_h^{j}(s,a)}} \\
\le & \frac{\sqrt{2\kappa_{\text{even}}m}}{\sqrt{\sum_{j \in \Gcal}N_h^{j}(s,a)}}
\end{align}
Thus
\begin{align}
\mathscr{A}_2 \le &
2\sum_{h=1}^H\EE_{d^{\tilde\pi}}\Bigg[\left(\frac{2}{\sqrt{\sum_{j \in \Gcal}\tilde N_h^{j}(s,a)}}
+ \frac{8\alpha m\sqrt{N_h^{\operatorname{cut}}(s,a)}}{\sum_{j \in \Gcal}\tilde N_h^{j}(s,a)}
 \right)H\sqrt{2\log\frac{2mSAH}{\delta}}
\\
&
\one\left\{N_h^{\Gcal,\cut_2}(s_h,a_h) > 0\right\}\Bigg] \\
\le &
2\sum_{h=1}^H\EE_{d^{\tilde\pi}}
\left[
\frac{\left(2+8\alpha\sqrt{2m}\right)\sqrt{\kappa_{\text{even}}}}{\sqrt{\sum_{j \in \Gcal}N_h^{j}(s,a)}}
H\sqrt{2\log\frac{2mSAH}{\delta}}
\one\left\{N_h^{\Gcal,\cut_2}(s_h,a_h) > 0\right\}
\right]\\
\le &
2\left(2+8\alpha\sqrt{2m}\right)\sqrt{\kappa_{\text{even}}}
H\sqrt{2\log\frac{2mSAH}{\delta}}
\sum_{h=1}^H\EE_{d^{\tilde\pi}}
\left[
\frac{\one\left\{N_h^{\Gcal,\cut_2}(s_h,a_h) > 0\right\}}{\sqrt{\sum_{j \in \Gcal}N_h^{j}(s,a)}}
\right]
\end{align}
Recall that $\Ccal_h = \left\{s | N_h^{\Gcal, \cut_2}(s,\tilde \pi(s)) >0\right\} $.
By Cauchy–Schwarz inequality and the definition of $\kappa$ in \pref{def:rel con num},
\begin{align}
\EE_{d^{\tilde\pi}}
\left[
\frac{\one\left\{N_h^{\Gcal,\cut_2}(s_h,a_h) > 0\right\}}{\sqrt{\sum_{j \in \Gcal}N_h^{j}(s,a)}}
\right]
\le &
\sqrt{
\EE_{d^{\tilde\pi}}
\left[
\frac{\one\left\{N_h^{\Gcal,\cut_2}(s_h,a_h) > 0\right\}}{{\sum_{j \in \Gcal}N_h^{j}(s,a)}}
\right]} \\
= &
\sqrt{
\sum_{s \in \Ccal_h}
\frac{d_h^{\tilde \pi}(s)}{{\sum_{j \in \Gcal}N_h^{j}(s,a)}}
}\\
= &
\sqrt{
\sum_{s \in \Ccal_h}
\frac{d_h^{\tilde \pi}(s)}{{\sum_{j \in \Gcal}N_h^{j}(s,a)}/\sum_{j\in\Gcal}K_j}
\frac{1}{\sum_{j\in\Gcal}K_j}
}\\
\le & \sqrt{
\sum_{s \in \Ccal_h}
\frac{\kappa}{\sum_{j\in\Gcal}K_j}
}
\le 
\sqrt{
\frac{\kappa S}{\sum_{j\in\Gcal}K_j}
}
\end{align}
In conclusion,
\begin{align}
\subopt(\hat \pi, \tilde \pi)
\le & \mathscr{A}_1 + \mathscr{A}_2\\
\le & 2H p^{\Gcal,0} 
+ 2\left(2+8\alpha\sqrt{2m}\right)\frac{\sqrt{\kappa\kappa_{\text{even}}S}}{\sqrt{\sum_{j\in\Gcal}K_j}}
H^2\sqrt{2\log\frac{2mSAH}{\delta}} \\
= & 2H p^{\Gcal,0}
+ O\left(
\sqrt{\kappa\kappa_{\text{even}}}H^2\sqrt{S}
\frac{1+\alpha\sqrt{m}}{\sqrt{\sum_{j\in\Gcal}K_j}}
\sqrt{\log\frac{mSAH}{\delta}}
\right)
\end{align}

\end{proof}

\section{Useful inequalities}
\begin{theorem}[Bernstein type of bound for martingale, Theorem 1.6 of \citep{freedman1975tail}]\label{thm:bernstein martingale}
Let $(\Omega, \mathcal{F}, P)$ be a probability triple.
Let $\mathcal{F}_0\subset\mathcal{F}_1\subset \cdots$ be an increasing sequence of sub-$\sigma$-fields of $\mathcal{F}$. 
Let $X_1, X_2, \ldots$ be random variables on $(\Omega, \mathcal{F}, P)$, such that $X_n$ is $\mathcal{F}_n$ measurable. 
Let $V_n = \VV\left[X_n | \mathcal{F}_{n-1}\right]$.
Assume $|X_n| \le 1$ and $\EE[X_n | \mathcal{F}_{n-1}]$ = 0.
Let
\begin{align}
&S_n = X_1 + \cdots + X_n\\
&T_n = V_1 + \cdots + V_n,
\end{align}
where $S_0 = T_0 = 0$.
Then, for any $a>0$, $b>0$,
\begin{equation}
\PP\left(|S_n| \ge a \mbox{  and  } T_n \le b \mbox{  for some $n$}   \right) 
\le 2\exp\left(-\frac{a^2}{2(a+b)}\right).
\end{equation}
\end{theorem}
By union bound and partition, we can get a more useful version of \pref{thm:bernstein martingale}.

We first present a result, which shows:
given,
\begin{equation}
\PP\left(X \ge t, Y\le t \right) \le \delta(t)
\end{equation}
We can bound 
$
\PP\left(X \ge Y \right)    
$
up to some error.
\begin{lemma}\label{lem:un bd prob}
Let $\{A_n\}_{n=1}^N$ and $\{B_n\}_{n=1}^N$ be two sequences of random variables. 
We don't make any assumption about the independence. 
Assume 
\begin{itemize}
\item $\forall n$, $0 \le B_n \le n M$ almost surely;
\item $\forall \delta > 0$, $f_\delta : \RR_+ \mapsto \RR_+$, $f_{\delta}(\cdot)$ monotonic increasing, 
\end{itemize}
If for all $t > 0$,
\begin{equation}\label{eq:prob under t}
\P\left(\bigcup_{n=1}^N \left\{|A_n| \ge f_{\delta}(t), B_n \le t\right\}\right) \le \delta
\end{equation}
Then for any $\epsilon > 0$,
\begin{equation}
\P\left(\bigcup_{n=1}^N \left\{|A_n| \ge f_{\delta}(B_n+ \epsilon) \right\}\right) \le NM\lceil 1/ \epsilon \rceil\delta
\end{equation}
\end{lemma}
\begin{proof}
See proof in \pref{sec:un bd prob}.
\end{proof}

\begin{corollary}\label{corr:bern var}
Under the assumption of \prettyref{thm:bernstein martingale}, suppose 
$X_n$ terminate at $n = N$.
Then, for all $0 < \delta < {2}{\exp(-2)}$, 
\begin{align}
& 
\PP\left(\bigcup_{n=1}^N\left\{|S_n| \ge 3\log \frac{2N}{\delta}    + \sqrt{2T_n\log\frac{2N}{\delta}}\right\}
\right)
\le \delta
\end{align}
\end{corollary}
\begin{proof}[Proof of \prettyref{corr:bern var}]
Let 
$\frac{\delta}{N} = 2\exp\left(-\frac{a^2}{2(a+b)}\right)$
then 
\begin{equation}
a = \log \frac{2N}{\delta}    + \sqrt{\log^2\frac{2N}{\delta} + 2b\log\frac{2N}{\delta}}
\end{equation}
by \prettyref{thm:bernstein martingale}, For all $b > 0$,
\begin{equation}
\PP\left(|S_n| \ge \log \frac{2N}{\delta}    + \sqrt{\log^2\frac{2N}{\delta} + 2b\log\frac{2N}{\delta}},
\mbox{ and }
T_n \le b \mbox{  for some $n$}
\right)    
\le \delta/N
\end{equation}
In \prettyref{lem:un bd prob}, let:
\begin{itemize}
\item $A_n = S_n$, $B_n = T_n$, $M=1$
\item $\epsilon = \frac{1}{2}\log\frac{2N}{\delta}$
\item $f_\delta(x) = \log \frac{2N}{\delta}    + \sqrt{\log^2\frac{2N}{\delta} + 2x\log\frac{2N}{\delta}}$
\end{itemize}
Because $0 < \delta < 2\exp(-2)$, $\epsilon \ge 1$.
then, we get:
\begin{align}
& \PP\left(\bigcup_{n=1}^N\left\{|S_n| \ge 3\log \frac{2N}{\delta}    + \sqrt{2T_n\log\frac{2N}{\delta}}\right\}
\right) \\
\le & \PP\left(\bigcup_{n=1}^N\left\{|S_n| \ge \log \frac{2N}{\delta}    + \sqrt{2\log^2\frac{2N}{\delta} + 2T_n\log\frac{2N}{\delta}}\right\}
\right)  \\
\le & N\lceil 1/ \epsilon \rceil\frac{\delta}{N} \le \delta
\end{align}
\end{proof}

\subsection{Proof for \prettyref{lem:un bd prob} \label{sec:un bd prob}}
\begin{proof}[Proof of \prettyref{lem:un bd prob}]
For discrete random variable, we can just conditioning on each possible value of $B_n$ and use a union bound. 
Here, because $B_n$ can be continuous random variable, we divide the range of $B_n$ into intervals.
And upper bound the target by law of total probability.

For all $n$, let:
\begin{equation}
0 < \frac{1}{\lceil 1/ \epsilon \rceil} < \frac{2}{\lceil 1/ \epsilon \rceil} 
< \cdots < 
\frac{nM\lceil 1/ \epsilon \rceil}{\lceil 1/ \epsilon \rceil} = nM
\end{equation}
Be a partition of interval $[0, nM]$.
Let $I_i := \left[\frac{i-1}{\lceil 1/ \epsilon \rceil}, \frac{i}{\lceil 1/ \epsilon \rceil}\right],i = 1, \ldots, nM\lceil 1/ \epsilon \rceil$ be a set of intervals.
Note that, $\bigcup_{i = 1}^{nM\lceil 1/ \epsilon \rceil}I_i = [0,nM]$.

Then
\begin{align}
&\bigcup_{n=1}^N \left\{|A_n| \ge f_{\delta}(B_n + \epsilon)\right\} =     
\bigcup_{n=1}^N \bigcup_{i=1}^{nM\lceil 1/ \epsilon \rceil}\left\{|A_n| \ge f_{\delta}(B_n+ \epsilon) , B_n \in I_i\right\}\\
= & \bigcup_{n=1}^N \bigcup_{i=1}^{nM\lceil 1/ \epsilon \rceil}\left\{|A_n| \ge f_{\delta}(B_n+ \epsilon) , 
\frac{i-1}{\lceil 1/ \epsilon \rceil} \le B_n 
\le \frac{i}{\lceil 1/ \epsilon \rceil}\right\} \\
\subseteq & \bigcup_{n=1}^N \bigcup_{i=1}^{nM\lceil 1/ \epsilon \rceil}\left\{|A_n| \ge f_{\delta}(\frac{i}{\lceil 1/ \epsilon \rceil}) , 
B_n \le \frac{i}{\lceil 1/ \epsilon \rceil}\right\} \\
\subseteq & \bigcup_{n=1}^N \bigcup_{i=1}^{NM\lceil 1/ \epsilon \rceil}\left\{|A_n| \ge f_{\delta}(\frac{i}{\lceil 1/ \epsilon \rceil}) , 
B_n \le \frac{i}{\lceil 1/ \epsilon \rceil}\right\} \\
= & \bigcup_{i=1}^{NM\lceil 1/ \epsilon \rceil}\bigcup_{n=1}^N \left\{|A_n| \ge f_{\delta}(\frac{i}{\lceil 1/ \epsilon \rceil}) , 
B_n \le \frac{i}{\lceil 1/ \epsilon \rceil}\right\}
\end{align}
Thus 
\begin{align}
\P\left(\bigcup_{n=1}^N \left\{|A_n| \ge f_{\delta}(B_n+ \epsilon) \right\}\right) 
\le &\sum_{i=1}^{NM\lceil 1/ \epsilon \rceil} \PP\left(\bigcup_{n=1}^N \left\{|A_n| \ge f_{\delta}(\frac{i}{\lceil 1/ \epsilon \rceil}) , 
B_n \le \frac{i}{\lceil 1/ \epsilon \rceil}\right\}\right)\\
\le& NM\lceil 1/ \epsilon \rceil\delta \quad \left(\mbox{By \pref{eq:prob under t}}\right)
\end{align}
\end{proof}

\end{document}